\documentclass[twoside,11pt]{article}

\usepackage[preprint]{jmlr2e}

\usepackage[utf8]{inputenc} 
\usepackage[T1]{fontenc}    
\usepackage{hyperref}       
\usepackage{url}            
\usepackage{booktabs}       
\usepackage{amsfonts}       
\usepackage{nicefrac}       
\usepackage{slashbox,multirow}
\usepackage{array,booktabs}
\newcolumntype{x}[1]{>{\centering\arraybackslash}p{#1}}

\usepackage{microtype}
\usepackage{graphicx}
\usepackage{amsmath,amssymb}
\usepackage{mathabx} 
\usepackage{caption}
\usepackage{subcaption}
\usepackage{booktabs} 
\usepackage{xspace}
\usepackage{enumerate}
\usepackage{mathrsfs}
\usepackage{etoolbox}

\usepackage{hyperref}
\usepackage{xcolor}

\usepackage{algorithm}
\usepackage{algcompatible}
\usepackage[noend]{algpseudocode}

\newboolean{showcomments}
\setboolean{showcomments}{true}

\newtheorem{definitionNew}{Definition}

\makeatletter
\renewenvironment{proof}[1][\proofname]{\par
\pushQED{\hfill\BlackBox}%
\normalfont \topsep6\p@\@plus6\p@\relax
\trivlist
\item\relax
{\itshape
#1\@addpunct{.}}\hspace\labelsep\ignorespaces
}{%
\popQED\endtrivlist\@endpefalse
}
\makeatother
\newcommand{\yuxin}[1]{\ifthenelse{\boolean{showcomments}}{\textcolor{blue}{[YC: #1]}}{}}
\newcommand{\adish}[1]{\ifthenelse{\boolean{showcomments}}{\textcolor{red}{AS: #1}}{}}

\newcommand{\commentout}[1]{}

\newcommand{\denselist}{
\itemsep -4pt\topsep-8pt\partopsep-8pt\itemindent 0pt
}

\newcommand{\Hypotheses}{\mathcal{H}}
\newcommand{\hypotheses}{H}

\newcommand{\hstar}{\hypothesis^\star}
\newcommand{\hinit}{{\hypothesis_0}}

\newcommand{\Examples}{\mathcal{Z}}
\newcommand{\Instances}{\mathcal{X}}
\newcommand{\Clabels}{\mathcal{Y}}

\newcommand{\examples}{Z}

\newcommand{\instances}{X}

\renewcommand{\example}{{z}}
\newcommand{\instance}{{x}}
\newcommand{\clabel}{{y}}

\newcommand{\TD}{\textsf{TD}\xspace}
\newcommand{\PBTD}{\textsf{PBTD}\xspace}
\newcommand{\NCTD}{\textsf{NCTD}\xspace}
\newcommand{\RTD}{\textsf{RTD}\xspace}

\newcommand{\VCD}{\textsf{VCD}\xspace}

\newcommand{\paren} [1] {\ensuremath{ \left( {#1} \right) }}

\def \argmin {\mathop{\rm arg\,min}}

\newcommand{\reals}{\ensuremath{\mathbb{R}}}

\newcommand{\hypothesis}[0]{\ensuremath{h}}

\newcommand{\bigO}[1]{\ensuremath{O\paren{#1}}}

\newcommand{\bigOmega}[1]{\ensuremath{\Omega\paren{#1}}}

\numberwithin{equation}{section}

\newcommand{\defref}[1]{Definition~\ref{#1}}
\newcommand{\tableref}[1]{Table~\ref{#1}}
\newcommand{\figref}[1]{Figure~\ref{#1}}

\newcommand{\secref}[1]{Section~\ref{#1}}

\newcommand{\thmref}[1]{Theorem~\ref{#1}}

\newcommand{\lemref}[1]{Lemma~\ref{#1}}
\renewcommand{\algref}[1]{Algorithm~\ref{#1}}

\newcommand{\lnref}[1]{Line~\ref{#1}}

\usepackage{subcaption}
\usepackage{wrapfig}

\newcommand{\pref}{\sigma}

\newcommand{\CF}{\textsf{CF}}

\newcommand{\Val}{D}
\newcommand{\Candidate}{\mathbf{C}}

\newcommand{\Teacher}{\mathrm{T}}

\newcommand{\TeachingSeq}{\mathcal{S}}

\newcommand{\Sigmatd}[1]{\Sigma_{#1}{\text -}\TD}

\newcommand{\SigmaCtd}[1]{\Sigma^c_{#1}{\text -}\TD}
\newcommand{\SigmaConst}{\Sigma_{\textsf{const}}}
\newcommand{\SigmaGlobal}{\Sigma_{\textsf{global}}}
\newcommand{\SigmaContGlobal}{\Sigma_{\textsf{global}}^c}
\newcommand{\SigmaGvs}{\Sigma_{\textsf{gvs}}}
\newcommand{\SigmaLocal}{\Sigma_{\textsf{local}}}
\newcommand{\SigmaLvs}{\Sigma_{\textsf{lvs}}}
\newcommand{\SigmaWinStayLoseShift}{\Sigma_{\textsf{wsls}}}

\newcommand{\gvs}{\textsf{gvs}}
\newcommand{\glbl}{\textsf{global}}
\newcommand{\lvs}{\textsf{lvs}}
\newcommand{\lcl}{\textsf{local}}

\newcommand{\LVS}{\textsf{lvs}}
\newcommand{\Local}{\textsf{local}}
\newcommand{\worstcase}{\textsf{wc}}
\newcommand{\SigmaConstTD}{\Sigmatd{\textsf{const}}}
\newcommand{\SigmaGlobalTD}{\Sigmatd{\textsf{global}}}
\newcommand{\SigmaContGlobalTD}{\SigmaCtd{\textsf{global}}}
\newcommand{\SigmaGvsTD}{\Sigmatd{\textsf{gvs}}}
\newcommand{\SigmaLocalTD}{\Sigmatd{\textsf{local}}}
\newcommand{\SigmaLvsTD}{\Sigmatd{\textsf{lvs}}}

\newcommand{\sigmacstar}{\sigma^{c\star}}

\usepackage{mathtools}
\DeclarePairedDelimiter{\ceil}{\lceil}{\rceil}

\newcommand{\hclass}[1]{H^{#1}}
\newcommand{\compactDSet}[1]{\Psi_{#1}}
\newcommand{\indexof}[1]{I{(#1)}}
\newcommand{\hnext}{\hypothesis_{\text{next}}}

\usepackage{tikz}
\usetikzlibrary{matrix,trees,backgrounds,fit,calc}

\def\constCircle{(4,0) circle (0.1cm)}

\def\gvsEllipse{(5,0) ellipse (2.2cm and 1.5cm)}
\def\localEllipse{(3,0) ellipse (2.2cm and 1.5cm)}
\def\lvsEllipse{(4,0) ellipse (4cm and 1.8cm)}

\newcommand{\distfun}{g}

\renewcommand{\cite}{\citep}

\jmlrheading{x}{2020}{x-xx}{10/2020}{xx/xx}{mansouri20}{Farnam Mansouri, Yuxin Chen, Ara Vartanian, Xiaojin Zhu, and Adish Singla}

\ShortHeadings{Preference-Based Batch and Sequential Teaching}{Mansouri, Chen, Vartanian, Zhu, and Singla}
\firstpageno{1}

\begin{document}
\title{Preference-Based Batch and Sequential Teaching\thanks{This manuscript is an extended version of the paper \citep{mansouri2019preference} that appeared in NeurIPS'19.}}


\author{\name Farnam Mansouri \email mfarnam@mpi-sws.org \\
       \addr Max Planck Institute for Software Systems (MPI-SWS)\\
       Saarbrucken, 66123, Germany\\
       \AND
       \name Yuxin Chen \email chenyuxin@uchicago.edu \\
       \addr University of Chicago\\
       Chicago, IL 60637, USA\\
       \AND
       \name Ara Vartanian \email aravart@cs.wisc.edu \\
       \addr University of Wisconsin-Madison\\
       Madison, WI 53706, USA\\
       \AND
       \name Xiaojin Zhu  \email jerryzhu@cs.wisc.edu \\
       \addr University of Wisconsin-Madison\\
       Madison, WI 53706, USA\\
       \AND
       \name Adish Singla \email adishs@mpi-sws.org \\
       \addr Max Planck Institute for Software Systems (MPI-SWS)\\
       Saarbrucken, 66123, Germany  
       }

\editor{x}

\maketitle

\newtoggle{longversion}
\settoggle{longversion}{true}
\begin{abstract}
Algorithmic machine teaching studies the interaction between a teacher and a learner where the teacher selects labeled examples aiming at teaching a target hypothesis. In a quest to lower teaching complexity, several teaching models and complexity measures have been proposed for both the batch settings (e.g., worst-case, recursive, preference-based, and non-clashing models) and the sequential settings (e.g., local preference-based model). To better understand the connections between these models, we develop a novel framework that captures the teaching process via preference functions $\Sigma$. In our framework, each function $\sigma \in \Sigma$ induces a teacher-learner pair with teaching complexity as $\TD(\sigma)$. We show that the above-mentioned teaching models are equivalent to specific types/families of preference functions. We analyze several properties of the teaching complexity parameter $\TD(\sigma)$ associated with different families of the preference functions, e.g.,  comparison to the VC dimension of the hypothesis class and additivity/sub-additivity of $\TD(\sigma)$ over disjoint domains. Finally, we identify preference functions inducing a novel family of sequential models with teaching complexity linear in the VC dimension: this is in contrast to the best-known complexity result for the batch models, which is quadratic in the VC dimension.

\end{abstract}
\section{Introduction}\label{sec:intro}
Algorithmic machine teaching studies the interaction between a teacher and a learner where the teacher’s goal is to find an optimal training sequence to steer the learner towards a target hypothesis \cite{goldman1995complexity,zilles2011models,zhu2013machine,singla2014near,zhu2015machine,DBLP:journals/corr/ZhuSingla18}. An important quantity of interest is the teaching dimension (\TD) of the hypothesis class, representing the number of examples needed to teach any hypothesis in a given class. 
Given that the teaching complexity depends on what assumptions are made about teacher-learner interactions, different teaching models lead to different notions of teaching dimension. In the past two decades, several such teaching models have been proposed, primarily driven by the motivation to lower teaching complexity and to find models for which the teaching complexity has better connections with learning complexity measured by Vapnik–Chervonenkis dimension (\VCD)~\cite{vapnik1971uniform} of the class.

One particularly well-established class of teaching models for machine teaching, among others, involves the version space learner. A learner in this model class maintains a version space (i.e., a subset of hypotheses that are consistent with the examples received from a teacher) and outputs a hypothesis from this version space. 
Most of the well-studied teaching models for version space learners are for the batch setting, e.g., worst-case~\cite{goldman1995complexity,kuhlmann1999teaching}, complexity-based~\cite{balbach2008measuring}, recursive~\cite{zilles2008teaching,zilles2011models,doliwa2014recursive}, preference-based~\cite{gao2017preference}, and non-clashing~\cite{pmlr-v98-kirkpatrick19a} models; see \secref{sec:non-seq-family-pref} for formal definitions of these models. In these batch models, the teacher first provides a set of examples to the learner and then the learner outputs a hypothesis. An optimal teacher under such batch settings does not have to adapt to the learner's hypothesis during the teaching process. In other words, the teacher can construct a complete sequence of examples of the minimal length before teaching begins.

In a quest to achieve more natural teacher-learner interactions and enable richer applications, various different models have been proposed for the sequential setting.  \citet{balbach2005teaching} studied teaching a variant of the version space learner restricted to incremental learning by introducing a neighborhood relation over hypotheses. It has been demonstrated that feedback about the learner's current hypothesis can be helpful when teaching such a learner in a sequential setting. \citet{chen2018understanding} recently studied the local preference-based model for version space learners, where the learner's choice of the next hypothesis depends on a preference function parametrized by the current hypothesis. It has been shown that the teacher could lower the teaching complexity significantly by adapting to the learner's current hypothesis for such a sequential learner~\cite{chen2018understanding}. Our teaching framework generalizes these existing models; see \secref{sec:seq-family-pref} for formal definitions and details.
%

Recently, teaching complexity results have been extended beyond version space learners, including models for gradient learners~\cite{liu2017iterative,DBLP:conf/icml/LiuDLLRS18,DBLP:conf/ijcai/KamalarubanDCS19}, models inspired by control theory~\cite{zhu2018optimal,DBLP:conf/aistats/LessardZ019}, models for sequential tasks~\cite{cakmak2012algorithmic,haug2018teaching,tschiatschek2019learner,icml20-adaptive-reward-poisoning,icml20-policy-teaching,zhang2020teaching}, and models for human-centered applications that require adaptivity~\cite{singla2013actively,hunziker2018teaching}. A recent line of research has studied robust notions of teaching in settings where the teacher has limited information about the learner's dynamics~\cite{dasgupta2019teaching,DBLP:conf/ijcai/DevidzeMH0S20,cicalese-icml20-teaching-with-limited}. We see these works as complementary to ours: we focus on version space learners with the teacher having full information about the learner, and aim to provide a unified framework for the batch and sequential teaching models.
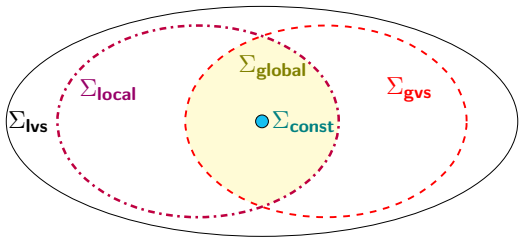
\begin{figure}[!t] 
\centering
\scalebox{.85}{
\begin{tikzpicture}
  \begin{scope}
    \clip \gvsEllipse;
    \fill[yellow!20] \localEllipse;
  \end{scope}
  \draw [fill=white!30!cyan] \constCircle node[text={rgb:green,3;blue,3}, right=.02cm] {\textbf{$\SigmaConst$}};
  \draw (3,0) node[text={rgb:red,7;green,7;blue,0},above=0.85cm,right=0.5cm] {\textbf{$\SigmaGlobal$}};
  \draw [red,line width=0.3mm,dashed] \gvsEllipse node[text={rgb:red,5;green,0;blue,0},above=.5cm,right=.8cm] {\textbf{$\SigmaGvs$}};
  \draw [purple,line width=0.4mm,dashdotted]\localEllipse node[text={rgb:red,1;purple,5;blue,2},above=.5cm,left=0.8cm] {\textbf{$\SigmaLocal$}};
  \draw \lvsEllipse node[text=black,left=3.2cm] {\textbf{$\SigmaLvs$}};
\end{tikzpicture}}
\caption{Venn diagram for different families of preference functions.} \label{fig:venndiagram_seq}
\end{figure}

\begin{table}[!t] 
\centering
\scalebox{0.9}{
\begingroup
\begin{tabular}{p{.175\linewidth}|p{.155\linewidth}|p{.19\linewidth}|p{.15\linewidth}|p{.14\linewidth}|p{.13\linewidth}}
    Families & $\SigmaConst$ & $\SigmaGlobal$ & $\SigmaGvs$ & $\SigmaLocal$ & $\SigmaLvs$\\ 
    \midrule
    Notion of \TD &  \worstcase-\TD & $\RTD$ / $\PBTD$ & $\NCTD$& 
    \Local-\PBTD & \LVS-\PBTD \\
    \hline
    Relation to \VCD & -- & $O(\VCD^2)$ & $O(\VCD^2)$ & $O(\VCD^2)$ & $O(\VCD)$\\
    \hline
    &\citet{goldman1995complexity}&\citet{zilles2011models,gao2017preference,DBLP:journals/corr/HuWLW17}&\citet{pmlr-v98-kirkpatrick19a}&\citet{chen2018understanding}&
\end{tabular}
\endgroup
}
\caption{Overview of our main results -- reduction to existing models and teaching complexity. 
}\label{tab:overview} 
\end{table}

\subsection{Overview of Main Results}\label{sec:intro.overview}

In this paper, we seek to gain a deeper understanding of how different teaching models relate to each other.
To this end, we develop a novel teaching framework that captures the teaching process via preference functions $\Sigma$. Here, a preference function $\sigma \in \Sigma$ models how a learner navigates in the version space as it receives teaching examples (see \secref{sec:model} for formal definition); in turn, each function $\sigma$ induces a teacher-learner pair with teaching dimension  $\TD(\sigma)$ (see \secref{sec:complexity}). We summarize some of the key results below:
\begin{itemize}
    \item We show that the well-studied teaching models in the batch setting, including the worst-case model \cite{goldman1995complexity}, the cooperative/recursive model \cite{zilles2011models}, the preference-based model \cite{gao2017preference} and the non-clashing model~\cite{pmlr-v98-kirkpatrick19a}, correspond to specific families of $\sigma$ functions in our framework. As a result, the teaching complexity for these models, namely the \emph{worst-case teaching dimension} (\worstcase-\TD)\footnote{In this paper, we refer to this classical notion of teaching dimension as \worstcase\textnormal{-}\TD (``\worstcase'' denoting worst-case model) instead of simply calling it \TD.}, the \emph{recursive teaching dimension} (\RTD), the \emph{preference-based teaching dimension} (\PBTD), and the \emph{no-clash teaching dimension} (\NCTD), correspond to the complexity of teaching specific families of batch learners under our framework
    (see \secref{sec:non-seq-family-pref} and \tableref{tab:overview}).
    \item We study the differences in the family of $\sigma$ functions inducing the strongest batch model~\cite{pmlr-v98-kirkpatrick19a} and functions inducing a weak sequential model~\cite{chen2018understanding}. The teaching complexity for the sequential model~\cite{chen2018understanding}, hereafter referred to as the \emph{local preference-based teaching dimension} (\Local-\PBTD), corresponds to the complexity of teaching specific family of sequential learners under our framework  (see \secref{sec:connection}, and the relationship between $\SigmaGvs$ and $\SigmaLocal$ in \figref{fig:venndiagram_seq}). 
    \item  We identify preference functions inducing a novel family of sequential models with teaching complexity linear in the \VCD of the hypothesis class. The preference functions in this family depend on both the learner's current hypothesis and the version space. Hereafter, we refer to the complexity of teaching such sequential models as the \emph{local version space preference-based teaching dimension} (\LVS-\PBTD). We provide a constructive procedure to find such $\sigma$ functions with low teaching complexity (\secref{sec:lvs_linvcd}). 
    \item 
    We analyze several important properties of the teaching complexity parameter $\TD(\sigma)$ associated with different families of the preference functions. In particular, we establish a lower bound on the teaching complexity $\TD(\sigma)$ w.r.t. \VCD for certain hypothesis classes, discuss the additivity/sub-additivity property of $\TD(\sigma)$ over disjoint domains, and compare the sizes of the different families of preference functions (\secref{sec:newresults}).
\end{itemize}

Our key findings are highlighted in \figref{fig:venndiagram_seq} and \tableref{tab:overview}. \figref{fig:venndiagram_seq} illustrates the relationship between different families of preference functions that we introduce, and \tableref{tab:overview} summarizes the key complexity results we obtain for different families. Although our main results are based on the setting where both the hypothesis class and the set of teaching examples are finite, we show that similar results could be extended to the infinite case, allowing us to establish our teaching complexity results as a generalization to \PBTD. Our unified view of the existing teaching models in turn opens up several intriguing new directions such as (i) using our constructive procedures to design preference functions for addressing open questions of whether \RTD / \NCTD is linear in \VCD, and (ii) understanding the notion of collusion-free teaching in sequential models. We discuss these directions further in \secref{sec:discussion}.


%
%
%




\section{The Teaching Model with Preference Functions}\label{sec:model}
\paragraph{The teaching domain.}
Let $\Instances$, $\Clabels$ be a ground set of unlabeled instances and the set of labels. 
Let $\Hypotheses$ be a finite class of hypotheses; each element $\hypothesis\in \Hypotheses$ is a function $\hypothesis: \Instances \rightarrow \Clabels$. Here, we only consider boolean functions and hence $\Clabels = \{0,1\}$. 
In our model, $\Instances$, $\Hypotheses$, and $\Clabels$ are known to both the teacher and the learner. There is a target hypothesis $\hstar\in \Hypotheses$ that is known to the teacher, but not the learner.  Let $\Examples \subseteq \Instances \times \Clabels$ be the ground set of labeled examples. Each element $\example = (\instance_\example,\clabel_\example) \in \Examples$ represents an example where the label is given by the target hypothesis $\hstar$, i.e., $\clabel_\example = \hstar(\instance_\example)$.
%
%
For any $\examples \subseteq \Examples$, the \emph{version space} induced by $\examples$ is the subset of hypotheses $\Hypotheses(\examples) \subseteq \Hypotheses$ that are consistent with the labels of all the examples, i.e., $\Hypotheses(\examples):= \{\hypothesis\in \Hypotheses \mid \forall \example = (\instance_\example, \clabel_\example) \in \examples, h(\instance_\example) = \clabel_\example\}$.

\paragraph{Learner's preference function.}
\looseness -1 We consider a generic model of the learner that captures our assumptions about how the learner adapts her hypothesis based on the examples received from the teacher. A key ingredient of this model is the learner's \emph{preference function} over the hypotheses. 
The learner, based on the information encoded in the inputs of preference function---which include the current hypothesis and the current version space---will choose one hypothesis in $\Hypotheses$. Our model of the learner strictly generalizes the local preference-based model considered in \cite{chen2018understanding}, where the learner's preference was only encoded by her current hypothesis.
Formally, we consider preference functions of the form $\sigma: \Hypotheses \times 2^\Hypotheses \times \Hypotheses \rightarrow \reals$.
For any two hypotheses $\hypothesis', \hypothesis''$, we say that the learner prefers $\hypothesis'$ to $\hypothesis''$ based on the current hypothesis $\hypothesis$ and version space $\hypotheses \subseteq \Hypotheses$, iff $\sigma(\hypothesis' ; \hypotheses, \hypothesis) < \sigma(\hypothesis'';\hypotheses, \hypothesis)$. If $\sigma(\hypothesis';\hypotheses, \hypothesis) = \sigma(\hypothesis'';\hypotheses, \hypothesis)$, then the learner could pick either one of these two. 
Note that many existing models of the learner could be viewed as special cases of such preference-based model with specific preference functions---e.g., when $\sigma(\hypothesis';\hypotheses, \hypothesis)$ is constant, the preference-based model reduces to the classical worst-case version space model as studied by \citet{goldman1995complexity}. We will discuss these special cases in detail in \secref{sec:non-seq-family-pref}. 

\paragraph{Interaction protocol and teaching objective.}
The teacher's goal is to steer the learner towards the target hypothesis $\hstar$ by providing a sequence of examples. The learner starts with an initial hypothesis $\hinit \in \Hypotheses$ before receiving any examples from the teacher. 
At time step $t$, the teacher selects an example $\example_t \in \Examples$, and the learner makes a transition from the current hypothesis to the next hypothesis.
Let us denote the examples received by the learner up to (and including) time step $t$ via $\examples_{t}$. Further, we denote the learner's version space at time step $t$ as $\hypotheses_t=\Hypotheses(\examples_{t})$, and the learner's hypothesis before receiving $\example_t$ as $\hypothesis_{t-1}$. The learner picks the next hypothesis based on the current hypothesis $\hypothesis_{t-1}$, version space $\hypotheses_{t}$, and preference function $\sigma$:
  \begin{align}
    \hypothesis_{t} \in \argmin_{\hypothesis'\in \hypotheses_{t}} \sigma(\hypothesis'; \hypotheses_t, \hypothesis_{t-1}).
    \label{eq.learners-jump}
  \end{align}

Upon updating the hypothesis $\hypothesis_t$, the learner sends $\hypothesis_t$ as feedback to the teacher. Teaching finishes here if the learner's updated hypothesis $\hypothesis_t$ equals $\hstar$. 
%
We summarize the interaction in Protocol~\ref{alg:interaction}. 
It is important to note that in our teaching model, the teacher and the learner use the same preference function. This assumption of shared knowledge of the preference function is also considered in existing teaching models for both the batch settings  (e.g., \cite{zilles2011models,gao2017preference}) and the sequential settings  (e.g., \cite{chen2018understanding}).




\makeatletter
\renewcommand{\ALG@name}{Protocol}
\makeatother

\begin{algorithm}[h!]
  \caption{Interaction protocol between the teacher and the learner}\label{alg:interaction}
    \begin{algorithmic}[1]
        \State learner's initial version space is $\hypotheses_{0} = \Hypotheses$ and learner starts from an initial hypothesis $\hinit \in \Hypotheses$
        \For{$t=1, 2, 3, \ldots$}
            \State learner receives $\example_t = (\instance_t, \clabel_t)$; updates  $\hypotheses_t=\hypotheses_{t-1} \cap \Hypotheses(\{\example_t\})$; picks $\hypothesis_t$ per Eq.~\eqref{eq.learners-jump};
            \State teacher receives $\hypothesis_t$ as feedback from the learner;
            \State \textbf{if \ } $\hypothesis_t = \hstar$ \textbf{then \ } teaching process terminates
        \EndFor
	\end{algorithmic}
\end{algorithm}




\section{The Complexity of Teaching with Preference Functions}\label{sec:complexity}
In this section, we formally state the notion of \emph{worst-case complexity} for teaching a preference-based learner. We first define the teaching complexity for a learner with a preference function from a given family. Then, we introduce an important family, namely \emph{collusion-free} preference functions, as the main focus of this paper.
\subsection{Teaching Dimension for a Family of Preference Functions}
\paragraph{Fixed preference function.}
Our objective is to design teaching algorithms that can steer the learner towards the target hypothesis in a minimal number of time steps. We study the \emph{worst-case} number of steps needed, as is common when measuring information complexity of teaching~\cite{goldman1995complexity,zilles2011models,gao2017preference,zhu2018optimal}. 
Fix the ground set of instances $\Instances$ and the learner's preference $\sigma$. For any version space $\hypotheses\subseteq\Hypotheses$, the worst-case optimal cost for steering the learner from $\hypothesis$ to $\hstar$ is characterized by
\vspace{1.5mm}
\begin{align*}
  \Val_{\pref}(\hypotheses, \hypothesis, \hstar) =
  \begin{cases}
    1, & 
    \exists z, \text{~s.t.~}\Candidate_{\sigma}(H,h,z) = \{\hstar\}\\
    1 + \min\limits_{\example} \max\limits_{\hypothesis'' \in \Candidate_\sigma(\hypotheses, \hypothesis, \example) 
    } \Val_{\pref}(\hypotheses \cap \Hypotheses(\{\example\}), \hypothesis'', \hstar) ,  &\text{otherwise}
  \end{cases}
\vspace{1.5mm}
\end{align*}
where 
    $\Candidate_\sigma(\hypotheses, \hypothesis, \example) =\argmin_{\hypothesis' \in \hypotheses \cap \Hypotheses(\{\example\})} \sigma(\hypothesis'; \hypotheses \cap \Hypotheses(\{\example\}), \hypothesis)$
denotes the set of candidate hypotheses most preferred by the learner. 
Note that our definition of teaching dimension is similar in spirit to the local preference-based teaching complexity defined by \citet{chen2018understanding}. We shall see in the next section, this complexity measure in fact reduces to existing notions of teaching complexity for specific families of preference functions.

Given a preference function $\pref$ and the learner's initial hypothesis $\hinit$, the teaching dimension w.r.t. $\pref$ is defined as the worst-case optimal cost for teaching any target $\hstar$:
\vspace{1.5mm}
\begin{align}\label{eq:sigmatd_fixedsigma}
    \TD_{\Instances,\Hypotheses,\hinit}(\pref) = \max_{\hstar} \Val_{\pref}(\Hypotheses, \hinit, \hstar).
\vspace{1.5mm}
\end{align}
\paragraph{Family of preference functions.}
In this paper, we will investigate several families of preference functions (as illustrated in \figref{fig:venndiagram_seq}).
For a family of preference functions $\Sigma$, 
we define the teaching dimension w.r.t the family $\Sigma$ as the teaching dimension w.r.t. the \emph{best} $\pref$ in that family:
\begin{align}\label{eq:sigmatd}
        \Sigma_{}{\text -}\TD_{\Instances,\Hypotheses,\hinit} =  \min_{\sigma\in \Sigma} \TD_{\Instances,\Hypotheses,\hinit}(\pref).
\end{align}

\subsection{Collusion-free Preference Functions}\label{sec:complexity:cf}
\looseness-1An important consideration when designing teaching models is to ensure that the teacher and the learner are ``collusion-free'', i.e., they are not allowed to collude or use some ``coding-trick'' to achieve arbitrarily low teaching complexity. 
%
A well-accepted notion of collusion-freeness in the batch setting is one proposed by \cite{goldman1996teaching} (also see \cite{angluin1997teachers,ott1999avoiding,pmlr-v98-kirkpatrick19a}). Intuitively, it captures the idea that a learner conjecturing hypothesis $\hypothesis$ will not change her mind when given additional information consistent with $\hypothesis$. 
In comparison to batch models, the notion of collusion-free teaching in the sequential models is not well understood. 
%
We introduce a novel notion of  collusion-freeness for the sequential setting, which captures the following idea: if $\hypothesis$ is the only hypothesis in the most preferred set defined by $\sigma$, then the learner 
    will always stay at $\hypothesis$ as long as additional information received by the learner is consistent with $\hypothesis$. We formalize this notion in the definition below. Note that for $\sigma$ functions corresponding to batch models (see \secref{sec:non-seq-family-pref}), \defref{def:seq-col-free} reduces to the collusion-free definition of \cite{goldman1996teaching}.

\begin{definitionNew}[Collusion-free preference]\label{def:seq-col-free}
%
Consider a time $t$ where the learner’s current hypothesis is $\hypothesis_{t-1}$ and version space is $\hypotheses_{t}$ (see Protocol \ref{alg:interaction}). Further assume that the learner’s preferred hypothesis for time $t$ is uniquely given by $\argmin_{\hypothesis' \in \hypotheses_t} \sigma(\hypothesis'; \hypotheses_t, \hypothesis_{t-1})=\{\hat{\hypothesis}\}$. 
Let $S$ be additional examples provided by an adversary from time $t$ onwards. We call a preference function collusion-free, if for any $S$ consistent with $\hat{\hypothesis}$, it holds that $\argmin_{\hypothesis' \in \hypotheses_t \cap \Hypotheses(S)} \sigma(\hypothesis'; \hypotheses_t \cap \Hypotheses(S), \hat{\hypothesis}) = \{\hat{\hypothesis}\}$.

%
%

\end{definitionNew}

In this paper, we study preference functions that are collusion-free. In particular, we use $\Sigma_{\CF}$ to denote the set of preference functions that induce collusion-free teaching: 
\begin{align*}
    \Sigma_{\CF} = \{\sigma\mid \sigma \text{ is collusion-free}\}. 
\end{align*}

Below, we provide two concrete examples for the collusion-free preference function families: 
\begin{enumerate}[(i)]
    \item \looseness-1``constant'' preference function family $\SigmaConst$ consists of functions where all the hypotheses in the current version space are preferred equally. Formally, this family is given by
    	$$\SigmaConst = \{\sigma \in  \Sigma_{\CF} \mid \exists c \in \reals, \text{~s.t.~} \forall \hypothesis', \hypotheses, \hypothesis, \sigma(\hypothesis'; \hypotheses, \hypothesis) = c\}.$$   	    
    Now, let us see why this family is collusion-free as per Definition~\ref{def:seq-col-free}. For any $\sigma \in \SigmaConst$, the only scenario where the learner's preferred hypothesis for time $t$ is given by $\argmin_{\hypothesis' \in \hypotheses_t} \sigma(\hypothesis'; \hypotheses_t, \hypothesis_{t-1})=\{\hat{\hypothesis}\}$, is when $\hypotheses_{t} = \{\hat{\hypothesis}\}$, i.e., there are no hypotheses left in the version space $\hypotheses_t$ other than $\hat{\hypothesis}$. Afterwards, by providing more examples consistent with $\hat{\hypothesis}$, the learner will stay on $\hat{\hypothesis}$. Thus, $\SigmaConst$ is a family of collusion-free preference functions. We will further discuss this family in \secref{sec:non-seq-family-pref}. 
    \item \looseness-1``win-stay-lose-shift'' preference function family $\SigmaWinStayLoseShift$ consists of functions where the learner prefers her current hypothesis as long as it stays consistent with the observed examples~\cite{bonawitz2014win,chen2018understanding}.  Formally, this family is given by 
    	$$\SigmaWinStayLoseShift = \{\sigma \in  \Sigma_{\CF} \mid \forall \hypothesis', \hypotheses, \hypothesis, \argmin_{\hypothesis' \in \Hypotheses} \pref(\hypothesis' ; \hypotheses, \hypothesis) = \{\hypothesis\} \}.$$
    It is easy to see why this family is collusion-free as per Definition~\ref{def:seq-col-free}. As per assumption in the definition, the learner will pick hypothesis $\hypothesis_t := \hat{\hypothesis}$ at time $t$. Afterwards, as long as the learner receives examples consistent with $\hat{\hypothesis}$, the desired condition in the definition holds for this family of functions.  It is important to note that $\sigma \in \SigmaWinStayLoseShift$ can depend on both the current hypothesis $\hypothesis_{t-1}$ and the version space $\hypotheses_t$, and the preferences play a role primarily when the current hypothesis becomes inconsistent. We further discuss this family in \secref{sec:lvs_linvcd}~and~\secref{sec:newresults:sub-add}.
\end{enumerate}

\section{Preference-based Batch Models} \label{sec:non-seq-family-pref}
In this section, we focus on preference functions that do not depend on the learner's current hypothesis. 
When teaching a learner with such a preference function, the teacher can construct an optimal sequence of examples in a batch before teaching begins. We study the complexity of teaching for different preference-based batch models and draw connections with well-established notions of teaching complexity in the literature.
\subsection{Families of Preference Functions}\label{sec:non-seq-family-pref.families}
We consider three families of preference functions which do not depend on the learner's current hypothesis. 
The first one, as already introduced in the previous section, is the family of constant preference functions $\SigmaConst$ given by:
\begin{align*}
    \SigmaConst =  \{\sigma \in  \Sigma_{\CF} \mid \exists c \in \reals, \text{~s.t.~} \forall \hypothesis', \hypotheses, \hypothesis, \sigma(\hypothesis'; \hypotheses, \hypothesis) = c\}.
\end{align*}
The second family, denoted by $\SigmaGlobal$, corresponds to the preference functions that do not depend on the learner's current hypothesis and version space. In other words, the preference functions capture some \emph{global} preference ordering of the hypotheses:
\begin{align*}
    \SigmaGlobal =  \{\sigma \in \Sigma_{\CF}\mid \exists~\distfun: \Hypotheses \rightarrow \reals, \text{~s.t.~}\forall \hypothesis', \hypotheses, \hypothesis,~\sigma(\hypothesis'; \hypotheses, \hypothesis) = \distfun(\hypothesis')\}.
\end{align*}
The third family, denoted by $\SigmaGvs$, corresponds to the preference functions that depend on the learner's version space, but do not depend on the learner's current hypothesis:
\begin{align*}
    \SigmaGvs =  \{\sigma \in \Sigma_{\CF} \mid \exists~\distfun: \Hypotheses\times 2^\Hypotheses \rightarrow \reals, \text{~s.t.~} \forall \hypothesis', \hypotheses, \hypothesis, \sigma(\hypothesis'; \hypotheses, \hypothesis) = \distfun(\hypothesis', \hypotheses)\}.
\end{align*}

\begin{figure}[!t]
\centering
\scalebox{.85}{
\begin{tikzpicture}
  \begin{scope}
    \clip \gvsEllipse;
    \fill[yellow!20] \localEllipse;
  \end{scope}
  \draw [fill=white!30!cyan] \constCircle node[text={rgb:green,3;blue,3}, right=.02cm] {\textbf{$\SigmaConst$}};
  \draw (3,0) node[text={rgb:red,10;green,10;blue,0},above=0.85cm,right=0.5cm] {\textbf{$\SigmaGlobal$}};
  \draw [red,line width=0.3mm,dashed] \gvsEllipse
  node[text={rgb:red,5;green,0;blue,0},above=.5cm,right=.8cm] {\textbf{$\SigmaGvs$}};
\end{tikzpicture}}
\caption{Batch models.} \label{fig:venndiagram_batch}
\end{figure}
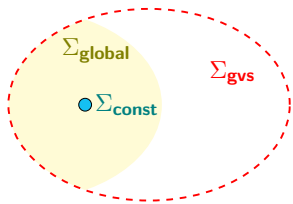

\figref{fig:venndiagram_batch} illustrates the relationship between these preference families. In \tableref{tab:batch_model_example_h2}, we provide a hypothesis class, as well as \emph{best} preference functions from the aforementioned three preference families (i.e., functions achieving minimal teaching complexity as per Eq.~\eqref{eq:sigmatd}). Specifically, the preference functions inducing the optimal teaching sequences/sets in \tableref{tab:h2} are given in Tables~\ref{tab:h2_pref_const}, \ref{tab:h2_pref_glbl}, and \ref{tab:h2_pref_gvs}. With these preference functions, one can derive the teaching complexity for this hypothesis class as $\SigmaConstTD=3$, $\SigmaGlobalTD=2$, and $\SigmaGvsTD=1$. Furthermore, in \secref{sec:seq-family-pref} we discuss Warmuth hypothesis class \cite{doliwa2014recursive} where $\SigmaConstTD=3$, $\SigmaGlobalTD=3$, and $\SigmaGvsTD=2$.

\begin{table}[h!]
    \centering
    \begin{subtable}[t]{\textwidth}
        \centering
        \begin{tabular}{c|cccccc||c|c|c}
        \backslashbox{$\Hypotheses$}{$\Instances$}
        & $\instance_1$ & $\instance_2$ & $\instance_3$ & $\instance_4$ & $\instance_5$ & $\instance_6$ & $\TeachingSeq_{\textsf{const}}$ & $\TeachingSeq_{\textsf{\glbl}}$ & $\TeachingSeq_{\gvs}$\\ 
        \hline
        $\hypothesis_1$ & 1 & 0 & 0 & 0 & 0 & 1 & $\paren{\instance_1, \instance_6}$ & $\paren{\instance_1, \instance_6}$  & $\paren{\instance_1}$\\
        $\hypothesis_2$ & 0 & 1 & 0 & 0 & 0 & 1 & $\paren{\instance_2, \instance_6}$ & $\paren{\instance_2, \instance_6}$ & $\paren{\instance_2}$\\
        $\hypothesis_3$ & 1 & 1 & 1 & 0 & 0 & 0 & $\paren{\instance_3, \instance_4, \instance_5}$ & $\paren{\instance_1}$ & $\paren{\instance_3}$\\
        $\hypothesis_4$ & 1 & 1 & 1 & 1 & 0 & 0 & $\paren{\instance_4, \instance_5}$ & $\paren{\instance_4, \instance_5}$ & $\paren{\instance_4}$\\
        $\hypothesis_5$ & 1 & 1 & 1 & 0 & 1 & 0 & $\paren{\instance_4, \instance_5}$ & $\paren{\instance_4, \instance_5}$ & $\paren{\instance_5}$\\
        $\hypothesis_6$ & 0 & 0 & 0 & 1 & 1 & 1 & $\paren{\instance_4, \instance_5}$ & $\paren{\instance_4, \instance_5}$ & $\paren{\instance_6}$
        \end{tabular}
        \caption{A hypothesis class and optimal teaching sequences/sets ($\TeachingSeq_{\textsf{const}}$, $\TeachingSeq_{\textsf{\glbl}}$, and $\TeachingSeq_{\gvs}$)  under different families of preference functions.}\label{tab:h2}
    \end{subtable}\\\vspace{3mm}
    \begin{subtable}[t]{.45\textwidth}
        \begin{tabular}{c|cccccc}
    $\hypothesis'$ & $\hypothesis_1$ & $\hypothesis_2$ & $\hypothesis_3$ & $\hypothesis_4$ & $\hypothesis_5$ & $\hypothesis_6$ \\\hline
        $\sigma_{\textsf{const}}(\hypothesis'; \cdot, \cdot)$ & 0 & 0 & 0 & 0 & 0 & 0\\
    \end{tabular}
    \caption{Preference function $\sigma_{\textsf{const}} \in \SigmaConst$} \label{tab:h2_pref_const}
    \end{subtable}
    \quad \qquad
    \begin{subtable}[t]{.45\textwidth}
        \begin{tabular}{c|cccccc}
    $\hypothesis'$ & $\hypothesis_1$ & $\hypothesis_2$ & $\hypothesis_3$ & $\hypothesis_4$ & $\hypothesis_5$ & $\hypothesis_6$ \\\hline
         $\sigma_{\textsf{global}}(\hypothesis'; \cdot, \cdot)$ & 1 & 1 & 0 & 1 & 1 & 1\\
    \end{tabular}
    \caption{Preference function $\sigma_{\glbl} \in \SigmaGlobal$}\label{tab:h2_pref_glbl}
    \end{subtable}
    \\\vspace{3mm}
    \begin{subtable}[t]{\textwidth}
    \begin{tabular}{c|cccccc}
    $\hypothesis'$ & $\hypothesis_1$ & $\hypothesis_2$ & $\hypothesis_3$ & $\hypothesis_4$ & $\hypothesis_5$ & $\hypothesis_6$ \\\hline
         \multirow{5}{*}{$\hypotheses$} & 
         $\{\hypothesis_1, \hypothesis_3, \hypothesis_4, \hypothesis_5\}$ & 
         $\{\hypothesis_2, \hypothesis_3, \hypothesis_4, \hypothesis_5\}$ &
         $\{\hypothesis_3, \hypothesis_4, \hypothesis_5\}$ & 
         $\{\hypothesis_4, \hypothesis_6\}$ & 
         $\{\hypothesis_5, \hypothesis_6\}$ & 
         $\{\hypothesis_1, \hypothesis_2, \hypothesis_6\}$ 
         \\ 
        & 
        $\{\hypothesis_1, \hypothesis_3, \hypothesis_4\}$ & 
        $\{\hypothesis_2, \hypothesis_3, \hypothesis_4\}$ &
         $\{\hypothesis_3, \hypothesis_4\}$ & 
         $\{\hypothesis_4\}$ & 
         $\{\hypothesis_5\}$ & 
         $\{\hypothesis_1, \hypothesis_6\}$ 
         \\
         & 
        $\{\hypothesis_1, \hypothesis_3, \hypothesis_5\}$ & 
        $\{\hypothesis_2, \hypothesis_3, \hypothesis_5\}$ &
         $\{\hypothesis_3, \hypothesis_5\}$ & 
          & 
          & 
         $\{\hypothesis_2, \hypothesis_6\}$ 
         \\
          & 
        $\{\hypothesis_1\}$ & 
        $\{\hypothesis_2\}$ &
         $\{\hypothesis_3\}$ & 
          & 
          & 
         $\{\hypothesis_6\}$ 
         \\\hline
         $\sigma_\gvs(\hypothesis'; \hypotheses, \cdot)$ & 0 & 0 & 0 & 0 & 0 & 0\\
    \end{tabular}
    \caption{Preference function $\sigma_\gvs \in \SigmaGvs$. For all other $\hypothesis', \hypotheses$ pairs not specified in the table, $\sigma(\hypothesis'; \hypotheses, \cdot) = 1$.}\label{tab:h2_pref_gvs}
    \end{subtable}
    \caption{A hypothesis class where $\SigmaConstTD=3$, $\SigmaGlobalTD=2$, and $\SigmaGvsTD=1$. The preference functions inducing optimal teaching sequences/sets in \tableref{tab:h2} (denoted by $\TeachingSeq_{\textsf{const}}$, $\TeachingSeq_{\textsf{\glbl}}$, and $\TeachingSeq_{\gvs}$) are specified in Tables~\ref{tab:h2_pref_const}, \ref{tab:h2_pref_glbl}, and \ref{tab:h2_pref_gvs}.
    }
    \label{tab:batch_model_example_h2}
\end{table}

\vspace{-1mm}
\subsection{Complexity Results}\label{sec:batch-complexity-finite}
We first provide several definitions, including the formal definition of the VC dimension and several existing notions of teaching dimension. The VC dimension captures the complexity notion for PAC learnability~\cite{blumer1989learnability} of a hypothesis class. Informally, it measures the capacity of a hypothesis class, i.e., characterizing how complicated and expressive a hypothesis class is in labeling the instances; we provide a rigorous definition below.


%
\begin{definitionNew}[Vapnik–Chervonenkis dimension \cite{vapnik1971uniform}]
The VC dimension for $H \subseteq \Hypotheses$ w.r.t. a fixed set of unlabeled instances $X \subseteq \Instances$, denoted by $\VCD(H,X)$, is the cardinality of the largest set of points $X' \subseteq X$ that are ``shattered''. Formally, let $\hypotheses_{|\instances} = \{(\hypothesis(\instance_1), ..., \hypothesis(\instance_n)) \ | \ \forall \hypothesis \in \hypotheses\}$ denote all possible patterns of $\hypotheses$ on $\instances$. Then $\VCD(H, X)=\max{|X'|}, \text{~s.t.~} X' \subseteq X \text{~and~} |H_{|X'}| = 2^{|X'|}$.\footnote{In the classical definition of $\VCD$, only the first argument $H$ is present; the second argument $X$ is omitted and is by default the ground set of unlabeled instances $\Instances$.} 
\end{definitionNew}

The concept of teaching dimension was
first introduced by \citet{goldman1995complexity}, measuring the minimum number of labeled instances a teacher must reveal to uniquely identify any target hypothesis; a formal definition is provided below.
\begin{definitionNew}[Teaching dimension \cite{goldman1995complexity}]\label{def:classicTD}
For any hypothesis $\hypothesis\in \Hypotheses$, we call a set of instances $\Teacher(\hypothesis) \subseteq \Instances$ a teaching set for $\hypothesis$, if it can uniquely identify $\hypothesis \in \Hypotheses$.  The teaching dimension for $\Hypotheses$, denoted by $\worstcase\textnormal{-}\TD(\Hypotheses)$, is the maximum size of the minimum teaching set for any $\hypothesis\in \Hypotheses$, i.e., $\worstcase\textnormal{-}\TD(\Hypotheses) = \max_{\hypothesis\in\Hypotheses} \min |\Teacher(\hypothesis)|$. Also, we refer to the teaching complexity of a fixed hypothesis $\hypothesis$ as $\worstcase\textnormal{-}\TD(\hypothesis, \Hypotheses) = \min |\Teacher(\hypothesis)|$.
\end{definitionNew}

As noted by \citet{zilles2008teaching}, the teaching dimension of \cite{goldman1995complexity} does not always capture the intuitive idea of cooperation between teacher and learner. The authors then introduced a model of cooperative teaching that resulted in the complexity notion of recursive teaching dimension, as defined below.
\begin{definitionNew}[Recursive teaching dimension 
\cite{zilles2008teaching,zilles2011models}]\label{def:rtd}
The recursive teaching dimension (\RTD) of $\Hypotheses$, denoted by $\RTD(\Hypotheses)$, is the
smallest number $k$, such that one can find an ordered sequence of hypotheses in $\Hypotheses$, denoted by $(\hypothesis_1, \dots, \hypothesis_i, \dots, \hypothesis_{|\Hypotheses|})$, where every hypothesis $\hypothesis_i$ has a teaching set of size no more than $k$ to be distinguished from the hypotheses in the remaining sequence.
\end{definitionNew}
%
%
%

In a recent work of  \cite{pmlr-v98-kirkpatrick19a}, a new notion of teaching complexity, called no-clash teaching dimension or \NCTD, was introduced (see definition below). Importantly, \NCTD is the optimal teaching complexity among teaching models in the batch setting that satisfy the collusion-free property of \cite{goldman1996teaching}. 
\begin{definitionNew}[No-clash teaching dimension \cite{pmlr-v98-kirkpatrick19a}]
Let $\Hypotheses$ be a hypothesis class and $\Teacher: \Hypotheses \rightarrow 2^\Instances$ be a ``teacher mapping''
on $\Hypotheses$, i.e., mapping a given hypothesis to a teaching set.\footnote{We refer the reader to the paper \cite{pmlr-v98-kirkpatrick19a} for a more formal description of ``teacher mapping''.}
We say that $\Teacher$ is non-clashing on $\Hypotheses$ iff there are no two distinct $\hypothesis,\hypothesis'\in\Hypotheses$ such that $\Teacher(\hypothesis)$ is consistent with $\hypothesis'$ and $\Teacher(\hypothesis')$ is consistent with $\hypothesis$.
The no-clash teaching dimension of $\Hypotheses$, denoted by $\NCTD(\Hypotheses)$, is defined as $\NCTD(\Hypotheses) = \min_{\Teacher \text{~is non-clashing}} \{\max_{h\in \Hypotheses} |\Teacher(h)|\}$.
\end{definitionNew}

We show in the following, that the teaching dimension $\Sigma{\text-}\TD$ in Eq.~\eqref{eq:sigmatd} unifies the above definitions of \TD's for batch models.
\begin{theorem}[Reduction to existing notions of TD's] \label{theorem:equivelence-results} Fix $\Instances,\Hypotheses,\hinit$. The teaching complexity for the three families reduces to the existing notions of teaching dimensions:
\begin{enumerate}\denselist
\item $\SigmaConstTD_{\Instances,\Hypotheses,\hinit} = \worstcase\textnormal{-}\TD(\Hypotheses)$
\item $\SigmaGlobalTD_{\Instances,\Hypotheses,\hinit} = \RTD(\Hypotheses) = O(\VCD(\Hypotheses, \Instances)^2) $
\item $\SigmaGvsTD_{\Instances,\Hypotheses,\hinit} = \NCTD(\Hypotheses) = O(\VCD(\Hypotheses, \Instances)^2)$
\end{enumerate}
\end{theorem}
Our teaching model strictly generalizes the local-preference based model of \cite{chen2018understanding}, which reduces to  the worst-case model when $\sigma \in \SigmaConst$~\cite{goldman1995complexity} and the recursive or global preference-based model when $\pref \in \SigmaGlobal$~\cite{zilles2008teaching,zilles2011models,gao2017preference,DBLP:journals/corr/HuWLW17}. Hence we get $\SigmaConstTD_{\Instances,\Hypotheses,\hinit} = \worstcase\textnormal{-}\TD(\Hypotheses)$ and $\SigmaGlobalTD_{\Instances,\Hypotheses,\hinit} = \RTD(\Hypotheses)$. 
To establish the equivalence between $\SigmaGvsTD_{\Instances,\Hypotheses,\hinit}$ and $\NCTD(\Hypotheses)$, it suffices to show that for any $\Instances,\Hypotheses,\hinit$, the following holds: (i) $\SigmaGvsTD_{\Instances,\Hypotheses,\hinit} \geq \NCTD(\Hypotheses)$, and (ii) $\SigmaGvsTD_{\Instances,\Hypotheses,\hinit} \leq \NCTD(\Hypotheses)$.
\iftoggle{longversion}
{The full proof is provided in Appendix~\ref{sec.appendix.batch-models}.}
{The full proof is provided in Appendix~A.2 of the supplementary.}
\subsection{Complexity Results: Extension to Infinite Domain}\label{sec:batch-complexity-infinite}
In this section, we extend our main results on the teaching complexity for batch models (\thmref{theorem:equivelence-results}) to the infinite domain. This allows us to additionally establish our teaching complexity results as a generalization to the preference-based teaching dimension (\PBTD) \cite{gao2017preference}. Note that \RTD is equivalent to \PBTD for a finite domain. We introduce the necessary notations and results here, and defer a more detailed presentation to Appendix~\ref{sec.appendix.batch-models-infinite}.


We begin by introducing the notation for an infinite set of instances as $\Instances^c$. 
Let $\Hypotheses^c$ to be an infinite class of hypotheses, where each $\hypothesis \in \Hypotheses^c$ is a function $\hypothesis : \Instances^c \to \Clabels$.  The preference functions for the infinite domain are given by $\sigma^c: \Hypotheses^c \times 2^{\Hypotheses^c} \times \Hypotheses^c \rightarrow \reals$. Similar to \defref{def:seq-col-free}, we consider the corresponding notion of collusion-free $\sigma^c$ for the infinite domain. Using this property, we use $\Sigma_{\CF}^c$ to denote the set of preference functions in the infinite domain that induces collusion-free teaching:
\begin{align*}
    \Sigma_{\CF}^c = \{\sigma^c\mid \sigma^c \text{ is collusion-free}\}.
\end{align*}


Similar to $\SigmaGlobal$ defined in Section~\ref{sec:non-seq-family-pref.families}, we now define the family of preference functions in an infinite domain that do not depend on the learner's current hypothesis and version space, given by:
\begin{align*}
    \SigmaContGlobal =  \{\sigma^c \in \Sigma_{\CF}^c \mid \exists~\distfun: \Hypotheses^c \rightarrow \reals, \text{~s.t.~}\forall \hypothesis', \hypotheses, \hypothesis,~\sigma^c(\hypothesis'; \hypotheses, \hypothesis) = \distfun(\hypothesis')\}
\end{align*}

Next we formally introduce the definition of \PBTD~\cite{gao2017preference}, which can be seen as extension of \RTD (\defref{def:rtd}) to infinite domains. Adapting the definitions from \cite{gao2017preference} to our notation,  we first consider preference relation, denoted as $\prec$, defined on $\Hypotheses^c$. We assume that $\prec$ is a strict partial order on $\Hypotheses^c$, i.e., $\prec$ is asymmetric and transitive. For every $\hypothesis \in \Hypotheses^c$, let ${\Hypotheses^c}_{\prec \hypothesis} = \{\hypothesis' \in \Hypotheses^c: \hypothesis' \prec \hypothesis\}$ be the set of hypothesis over which $\hypothesis$ is strictly preferred. 

%
%


%


\begin{definitionNew}[
Preference-based teaching dimension (based on \citet{gao2017preference})]
For $\hypothesis \in \Hypotheses$, and a preference relation $\prec$, we define the following measures:
\begin{itemize}
    \item $\PBTD(\hypothesis, \Hypotheses^c, \prec) = \worstcase\textnormal{-}\TD(\hypothesis, \Hypotheses^c \setminus {\Hypotheses^c}_{\prec \hypothesis})$ where $\worstcase\textnormal{-}\TD(\hypothesis, \cdot)$ is based on Definition~\ref{def:classicTD}. 
    \item $\PBTD(\Hypotheses^c, \prec) = \sup_{\hypothesis \in \Hypotheses^c} \PBTD(\hypothesis, \Hypotheses^c, \prec)$.
    \item $\PBTD(\Hypotheses^c) = \inf_{\prec \text{~{is a strict partial order on} } \Hypotheses^c} \PBTD(\Hypotheses^c, \prec)$    
\end{itemize}
\end{definitionNew}
As an extension of \thmref{theorem:equivelence-results} to infinite domains, we show in the following theorem that $\SigmaContGlobalTD_{\Instances^c,\Hypotheses^c,\hinit}$ is equivalent to \PBTD.

\begin{theorem}[Reduction to PBTD] \label{theorem:equivelence-results-infinite} Fix $\Instances^c,\Hypotheses^c,\hinit$.  Assume that for any strict partial order $\prec$ on $\Hypotheses^c$, there exists a function $g: \Hypotheses^c \to \reals$ such that for any two hypothesis $\hypothesis'' \neq \hypothesis'$, if $\hypothesis'' \prec \hypothesis'$ we have $g(\hypothesis'') > g(\hypothesis')$. Then, the teaching complexity for the $\SigmaContGlobal$ family reduces to the existing notion of \PBTD, i.e.,  $\SigmaContGlobalTD_{\Instances^c,\Hypotheses^c,\hinit} = \PBTD(\Hypotheses^c)$.
\end{theorem}

The proof is given in Appendix~\ref{sec.appendix.batch-models-infinite}. The results in Theorem~\ref{theorem:equivelence-results} and Theorem~\ref{theorem:equivelence-results-infinite} combined show that our preference-based batch models are equivalent to the existing notions of teaching dimensions (\worstcase-\TD, \RTD, \PBTD, and \NCTD). In the next section, we will consider preference-based sequential models.
 

 




\section{Preference-based Sequential Models} \label{sec:seq-family-pref}
In this section, we introduce two families of sequential preference functions that depend on the learner's current hypothesis. We establish connections between the complexity of teaching such sequential models with that of the aforementioned batch models, as well as with the VC dimension.

\subsection{Families of Preference Functions}
\looseness -1 
We investigate two families of preference functions that depend on the learner's current hypothesis $\hypothesis_{t-1}$.
The first one is the family of local preference-based functions \cite{chen2018understanding}, denoted by $\SigmaLocal$, which corresponds to preference functions that depend on the learner's current hypothesis, but do not depend on the learner's version space:
\begin{align*}
    \SigmaLocal =  \{\sigma \in \Sigma_{\CF} \mid 
    \exists~\distfun: \Hypotheses\times \Hypotheses \rightarrow \reals, 
    \text{~s.t.~} \forall \hypothesis', \hypotheses, \hypothesis,  \sigma(\hypothesis'; \hypotheses, \hypothesis) = \distfun(\hypothesis', \hypothesis)\}
\end{align*}
The second family, denoted by $\SigmaLvs$, corresponds to the preference functions that depend on 
all three arguments of $\sigma(\hypothesis'; \hypotheses, \hypothesis)$. 
The dependence of $\sigma$ on the learner's current hypothesis and the version space renders a powerful family of preference functions:
\begin{align*}
    \SigmaLvs =  \{\sigma \in \Sigma_{\CF} \mid \exists~\distfun: \Hypotheses \times 2^\Hypotheses \times \Hypotheses \rightarrow \reals, \text{~s.t.~} \forall \hypothesis', \hypotheses, \hypothesis,\sigma(\hypothesis'; \hypotheses, \hypothesis) = \distfun(\hypothesis',\hypotheses, \hypothesis)\}
\end{align*}

\looseness -1 \figref{fig:venndiagram_seq} illustrates the relationship between these preference families. 
In \tableref{tab:warmuth_class_sequence_pref}, we provide an example of the Warmuth hypothesis class~\cite{doliwa2014recursive}\footnote{The Warmuth hypothesis class is the smallest class for which \RTD exceeds \VCD.}, as well as \emph{best} preference functions from the aforementioned batch and sequential preference families (i.e., functions achieving minimal teaching complexity as per Eq.~\eqref{eq:sigmatd}). Specifically, the preference functions inducing the optimal teaching sequences in \tableref{tab:app:warmth_example_sequences} are given in Tables~\ref{tab:app:pref_const}, \ref{tab:app:pref_global}, \ref{tab:app:pref_gvs}, \ref{tab:app:pref_local}, and \ref{tab:app:pref_lvs}. With these preference functions, one can derive the teaching complexity for this hypothesis class as $\SigmaConstTD=3$, $\SigmaGlobalTD=3$, $\SigmaGvsTD=2$, $\SigmaLocalTD=2$, and $\SigmaLvsTD=1$.

\begin{table}[t!]
\centering
    \begin{subtable}[t]{\textwidth}
        \centering
        \begin{tabular}{l|lllll||l|l|l|l}
        \backslashbox{$\Hypotheses$}{$\Instances$} & $\instance_1$ & 
        $\instance_2$ & 
        $\instance_3$ & 
        $\instance_4$ & 
        $\instance_5$ & 
        $\TeachingSeq_{\textsf{const}}=\TeachingSeq_{\textsf{\glbl}}$ &
        $\TeachingSeq_{\gvs}$ &
        $\TeachingSeq_{\lcl}$ &
        $\TeachingSeq_{\lvs}$ 
        \\ \hline
        $\hypothesis_1$ 
        & 1 & 1 & 0  & 0 & 0 
        & $\paren{\instance_1, \instance_2, \instance_4}$ 
        & $\paren{\instance_1, \instance_2}$
        & $\paren{\instance_1}$ 
        & $\paren{\instance_1}$ 
        \\
        $\hypothesis_2$
        & 0 & 1 & 1 & 0 & 0
        & $\paren{\instance_2, \instance_3, \instance_5}$ 
        & $\paren{\instance_2, \instance_3}$
        & $\paren{\instance_3}$ 
        & $\paren{\instance_2}$ 
        \\
        $\hypothesis_3$
        & 0 & 0 & 1 & 1 & 0
        & $\paren{\instance_1, \instance_3, \instance_4}$ 
        & $\paren{\instance_3, \instance_4}$
        & $\paren{\instance_3, \instance_4}$  
        & $\paren{\instance_3}$ 
        \\
        $\hypothesis_4$
        & 0 & 0 & 0 & 1 & 1
        & $\paren{\instance_2, \instance_4, \instance_5}$ 
        & $\paren{\instance_4, \instance_5}$
        & $\paren{\instance_5, \instance_4}$  
        & $\paren{\instance_4}$ 
        \\
        $\hypothesis_5$
        & 1 & 0 & 0 & 0 & 1
        & $\paren{\instance_1, \instance_3, \instance_5}$ 
        & $\paren{\instance_1, \instance_5}$
        & $\paren{\instance_5}$ 
        & $\paren{\instance_5}$ 
        \\
        $\hypothesis_6$
        & 1 & 1 & 0 & 1 & 0
        & $\paren{\instance_1, \instance_2, \instance_4}$ 
        & $\paren{\instance_2, \instance_4}$
        & $\paren{\instance_4}$ 
        & $\paren{\instance_3}$ 
        \\
        $\hypothesis_7$
        & 0 & 1 & 1 & 0 & 1
        & $\paren{\instance_2, \instance_3, \instance_5}$ 
        & $\paren{\instance_3, \instance_5}$
        & $\paren{\instance_3, \instance_5}$  
        & $\paren{\instance_4}$ 
        \\
        $\hypothesis_8$
        & 1 & 0 & 1 & 1 & 0
        & $\paren{\instance_1, \instance_3, \instance_4}$ 
        & $\paren{\instance_1, \instance_4}$
        & $\paren{\instance_4, \instance_3}$  
        & $\paren{\instance_5}$ 
        \\
        $\hypothesis_9$
        & 0 & 1 & 0 & 1 & 1
        & $\paren{\instance_2, \instance_4, \instance_5}$ 
        & $\paren{\instance_2, \instance_5}$
        & $\paren{\instance_4, \instance_5}$  
        & $\paren{\instance_1}$ 
        \\
        $\hypothesis_{10}$
        & 1 & 0 & 1 & 0 & 1 
        & $\paren{\instance_1, \instance_3, \instance_5}$ 
        & $\paren{\instance_1, \instance_3}$
        & $\paren{\instance_5, \instance_3}$ 
        & $\paren{\instance_2}$ 
        \end{tabular}
        \vspace{-1mm}
        \caption{The Warmuth hypothesis class \cite{doliwa2014recursive} and optimal teaching sequences ($\TeachingSeq_{\textsf{const}}$, $\TeachingSeq_{\textsf{\glbl}}$, $\TeachingSeq_{\gvs}$, $\TeachingSeq_{\lcl}$, and $\TeachingSeq_{\lvs}$)  under different families of preference functions.}\label{tab:app:warmth_example_sequences}
    \end{subtable}
    \\\vspace{4mm}
    \begin{subtable}[t]{.45\textwidth}
        \centering
        \scalebox{0.85}{
        \begin{tabular}{c|c}
        $\hypothesis'$ & $\forall \hypothesis' \in \Hypotheses$  \\\hline
        $\sigma_{\textsf{const}}(\cdot; \cdot, \cdot)$ & 0
    \end{tabular}}
    \vspace{-1mm}
   \caption{$\sigma_{\textsf{const}}(\cdot; \cdot, \cdot)$}
    \label{tab:app:pref_const}
    \end{subtable}
    \quad
    \begin{subtable}[t]{.45\textwidth}
        \centering
        \scalebox{0.85}{
        \begin{tabular}{c|c}
        $\hypothesis'$ & $\forall \hypothesis' \in \Hypotheses$  \\\hline
        $\sigma_{\glbl}(\hypothesis'; \cdot, \cdot)$ & 0\\
    \end{tabular}}
    \vspace{-1mm}
   \caption{$\sigma_{\glbl}(\hypothesis'; \cdot, \cdot)$}
    \label{tab:app:pref_global}
    \end{subtable}
        \\\vspace{4mm}
        \begin{subtable}[t]{\textwidth}
    \scalebox{0.85}{
    \begin{tabular}{c|cccccccccc}
    $\hypothesis'$ & $\hypothesis_1$ & $\hypothesis_2$ & $\hypothesis_3$ & $\hypothesis_4$ & $\hypothesis_5$ & $\hypothesis_6$ & $\hypothesis_7$ & $\hypothesis_8$ &  $\hypothesis_9$ & $\hypothesis_{10}$\\\hline
         \multirow{2}{*}{$\hypotheses$} & 
         $\{\hypothesis_1, \hypothesis_6\}$ & 
         $\{\hypothesis_2, \hypothesis_7\}$ &
         $\{\hypothesis_3, \hypothesis_8\}$ & 
         $\{\hypothesis_4, \hypothesis_9\}$ & 
         $\{\hypothesis_5, \hypothesis_{10}\}$ &
         $\{\hypothesis_6, \hypothesis_9\}$ & 
         $\{\hypothesis_7, \hypothesis_{10}\}$ & 
         $\{\hypothesis_8, \hypothesis_6\}$ & 
         $\{\hypothesis_9, \hypothesis_7\}$ & 
         $\{\hypothesis_{10}, \hypothesis_8\}$ \\
        & 
         $\{\hypothesis_1\}$ & 
         $\{\hypothesis_2\}$ &
         $\{\hypothesis_3\}$ & 
         $\{\hypothesis_4\}$ & 
         $\{\hypothesis_5\}$ &
         $\{\hypothesis_6\}$ & 
         $\{\hypothesis_7\}$ & 
         $\{\hypothesis_8\}$ & 
         $\{\hypothesis_9\}$ & 
         $\{\hypothesis_{10}\}$ \\\hline
         $\sigma_\gvs$ & 
         0 & 0 & 0 & 0 & 0 & 0 & 0 & 0 & 0 & 0 \\
    \end{tabular}}
    \vspace{-1mm}
    \caption{$\sigma_\gvs(\hypothesis'; \hypotheses, \cdot)$. For all other $\hypothesis', \hypotheses$ pairs not specified in the table, $\sigma_\gvs(\hypothesis'; \hypotheses, \cdot) = 1$.}\label{tab:app:pref_gvs}
    \end{subtable}
        \\\vspace{4mm}
    \begin{subtable}[t]{.7\textwidth}
        \centering
        \scalebox{0.85}{
        \begin{tabular}{c|p{.15cm}p{.15cm}p{.15cm}p{.15cm}p{.15cm}p{.15cm}p{.15cm}p{.15cm}p{.15cm}p{.15cm}}
        {$\hypothesis'$} & 
        $\hypothesis_1$ & $\hypothesis_2$ & $\hypothesis_3$ & $\hypothesis_4$ & $\hypothesis_5$ & $\hypothesis_6$ & $\hypothesis_7$ & $\hypothesis_8$ & $\hypothesis_9$ & $\hypothesis_{10}$ \\\hline
        $\sigma_{\lcl}(\hypothesis'; \cdot, \hypothesis=\hypothesis_1)$ & 
        0 & 2 & 4 & 4 & 2 & 1 & 3 & 3 & 3 & 3 \\
        {\footnotesize $\dots$}&  
    \end{tabular}}
    \vspace{-1mm}
    \caption{$\sigma_{\lcl}(\hypothesis'; \cdot, \hypothesis)$ 
    representing the Hamming distance between $\hypothesis'$ and $\hypothesis$. } \label{tab:app:pref_local}
    \end{subtable}
             \\\vspace{4mm}
     \begin{subtable}[t]{\textwidth}
        \centering
        \scalebox{0.85}{
        \begin{tabular}{c|cc|cc|cc|cc|cc}
        $\hypothesis'$ & 
         \multicolumn{2}{c|}{$\hypothesis_1$} & 
         \multicolumn{2}{c|}{$\hypothesis_2$} & 
         \multicolumn{2}{c|}{$\hypothesis_3$} & 
         \multicolumn{2}{c|}{$\hypothesis_4$} & 
         \multicolumn{2}{c}{$\hypothesis_5$}
         \\\hline
        \multirow{2}{*}{$\hypotheses$} & \multicolumn{2}{c|}{$\{\hypothesis_1\} \cup$} & \multicolumn{2}{c|}{$\{\hypothesis_2\} \cup$} &
         \multicolumn{2}{c|}{$\{\hypothesis_3\} \cup$} &
         \multicolumn{2}{c|}{$\{\hypothesis_4\} \cup$} &
         \multicolumn{2}{c}{$\{\hypothesis_5\} \cup$}
         \\
        &\multicolumn{2}{c|}{$\{\hypothesis_5, \hypothesis_6,\hypothesis_8,\hypothesis_{10}\}^*$}
        &\multicolumn{2}{c|}{$\{\hypothesis_1, \hypothesis_7,\hypothesis_6,\hypothesis_{9}\}^*$}
        &\multicolumn{2}{c|}{$\{\hypothesis_2, \hypothesis_7,\hypothesis_8,\hypothesis_{10}\}^*$}
        &\multicolumn{2}{c|}{$\{\hypothesis_3, \hypothesis_6,\hypothesis_8,\hypothesis_{9}\}^*$}
        &\multicolumn{2}{c}{$\{\hypothesis_4, \hypothesis_7,\hypothesis_9,\hypothesis_{10}\}^*$}
         \\\hline
         $\hypothesis$ & 
         \multicolumn{2}{c|}{$\hypothesis_1$} & 
         $\hypothesis_1$ & $\hypothesis_2$ &
         $\hypothesis_1$ & $\hypothesis_3$ &
         $\hypothesis_1$ & $\hypothesis_4$ &
         $\hypothesis_1$ & $\hypothesis_5$
         \\\hline
        $\sigma_\lvs$ & 
        \multicolumn{2}{c|}{0} & 
        0 & 0 & 0 & 0 & 0 & 0 & 0 & 0\\
        \end{tabular}
        }\\
        $\vdots$\\
        \scalebox{0.85}{
        \begin{tabular}{c|cc|cc|cc|cc|cc}
         $\hypothesis'$ & 
         \multicolumn{2}{c|}{$\hypothesis_6$} & 
         \multicolumn{2}{c|}{$\hypothesis_7$} & 
         \multicolumn{2}{c|}{$\hypothesis_8$} & 
         \multicolumn{2}{c|}{$\hypothesis_9$} & 
         \multicolumn{2}{c}{$\hypothesis_{10}$}
         \\\hline
                 \multirow{2}{*}{$\hypotheses$} & \multicolumn{2}{c|}{$\{\hypothesis_6\} \cup$} & \multicolumn{2}{c|}{$\{\hypothesis_7\} \cup$} &
         \multicolumn{2}{c|}{$\{\hypothesis_8\} \cup$} &
         \multicolumn{2}{c|}{$\{\hypothesis_9\} \cup$} &
         \multicolumn{2}{c}{$\{\hypothesis_{10}\} \cup$}
         \\ 
          &
         \multicolumn{2}{c|}{$\{\hypothesis_1, \hypothesis_4,\hypothesis_5,\hypothesis_{9}\}^*$}
        &\multicolumn{2}{c|}{$\{\hypothesis_1, \hypothesis_2,\hypothesis_5,\hypothesis_{10}\}^*$}
        &\multicolumn{2}{c|}{$\{\hypothesis_1, \hypothesis_2,\hypothesis_3,\hypothesis_{6}\}^*$}
        &\multicolumn{2}{c|}{$\{\hypothesis_2, \hypothesis_3,\hypothesis_4,\hypothesis_{7}\}^*$}
        &\multicolumn{2}{c}{$\{\hypothesis_3, \hypothesis_4,\hypothesis_5,\hypothesis_{8}\}^*$}
         \\\hline
         $\hypothesis$ & 
         $\hypothesis_1$ & $\hypothesis_6$ &
         $\hypothesis_1$ & $\hypothesis_7$ &
         $\hypothesis_1$ & $\hypothesis_8$ &
         $\hypothesis_1$ & $\hypothesis_9$ &
         $\hypothesis_1$ & $\hypothesis_{10}$ \\\hline
        $\sigma_\lvs$ & 
        0 & 0 & 
        0 & 0 & 0 & 0 & 0 & 0 & 0 & 0\\
    \end{tabular}}
    \vspace{-1mm}
    \caption{$\sigma_\lvs(\hypothesis'; \hypotheses, \hypothesis)$. 
    Here, $\{\cdot \}^*$ denotes all subsets. For all other triplets not specified, $\sigma_\lvs(\hypothesis'; \hypotheses, \hypothesis) = 1$.}\label{tab:app:pref_lvs}
    \end{subtable}
\caption{Warmuth hypothesis class \cite{doliwa2014recursive} where $\SigmaConstTD=3$, $\SigmaGlobalTD=3$, $\SigmaGvsTD=2$, $\SigmaLocalTD=2$, and $\SigmaLvsTD=1$. The preference functions inducing optimal teaching sequences in \tableref{tab:app:warmth_example_sequences} (denoted by $\TeachingSeq$) are specified in Tables~\ref{tab:app:pref_const}, \ref{tab:app:pref_global}, \ref{tab:app:pref_gvs}, \ref{tab:app:pref_local}, and \ref{tab:app:pref_lvs}.
}
\label{tab:warmuth_class_sequence_pref}
\end{table}




\subsection{Comparing $\SigmaGvsTD$ and $\SigmaLocalTD$}
\label{sec:connection}


In the following, we show that substantial differences arise as we transition from $\sigma$ functions inducing the strongest batch (i.e., non-clashing) model to $\sigma$ functions inducing a weak sequential (i.e., local preference-based) model. 

\begin{theorem}\label{thm:local-eq-GVS}
Neither of the families $\SigmaGvs$ and $\SigmaLocal$ dominates the other. Specifically,
\begin{enumerate}\denselist
    \item $\SigmaGvs\cap \SigmaLocal=\SigmaGlobal$
    \item There exist $\Hypotheses$, $\Instances$, where $\forall \hinit\in \Hypotheses, 
    \SigmaLocalTD_{\Instances,\Hypotheses,\hinit} > \SigmaGvsTD_{\Instances,\Hypotheses,\hinit}$    
    \item There exist $\Hypotheses$, $\Instances$, where $\forall \hinit\in \Hypotheses, 
    \SigmaLocalTD_{\Instances,\Hypotheses,\hinit} < \SigmaGvsTD_{\Instances,\Hypotheses,\hinit}$, and this gap can be made arbitrarily large.
\end{enumerate}
\end{theorem}

\begin{proof}[Proof Sketch of \thmref{thm:local-eq-GVS}]
\textbf{Part 1}: The proof is based on the observation that the input domains between $\sigma_\lcl \in \SigmaLocal$ and $\sigma_\gvs \in \SigmaGvs$ overlap at the domain of the first argument, which is the one taken by $\sigma_\glbl \in \SigmaGlobal$. Therefore, $\forall \sigma \in \SigmaGlobal, \sigma \in \SigmaGvs \cap \SigmaLocal$. This intuition is formalized as a proof in Appendix~\ref{sec.appendix.seq-models_globalVS-local.part1}.

\looseness-1\textbf{Part 2}: We first identify $\Hypotheses$, $\Instances$, $\hinit$, where $\SigmaGvsTD_{\Instances,\Hypotheses,\hinit} = 1$ and $\SigmaGlobalTD_{\Instances,\Hypotheses,\hinit} = 2$. 
\tableref{tab:batch_model_example_h2} illustrates such a class.
Here, since $\SigmaGlobalTD_{\Instances,\Hypotheses,\hinit}  = 2$,
then by \lemref{lem:local-rtd-1} proven in Appendix~\ref{sec.appendix.seq-models_globalVS-local.part2}, it must hold that $\SigmaLocalTD_{\Instances,\Hypotheses,\hinit} > 1 = \SigmaGvsTD_{\Instances,\Hypotheses,\hinit}$.

\textbf{Part 3}: To prove Part 3, we consider the powerset hypothesis class of size $7\cdot2^m$ for any positive integer $m$, and show that the gap $\SigmaGvsTD_{\Instances,\Hypotheses,\hinit} - \SigmaLocalTD_{\Instances,\Hypotheses,\hinit} \geq 2^{m-1}$. 
 In particular, in the earlier version of this paper \cite{mansouri2019preference}, we showed that for the powerset hypothesis class of size $7$,  $\SigmaGvsTD_{\Instances,\Hypotheses,\hinit} \geq 4$ and $\SigmaLocalTD_{\Instances,\Hypotheses,\hinit} \leq 3$. 
 Based on this result, we then provide a constructive procedure that extends the gap w.r.t. $m$ when considering the powerset hypothesis class of size $7\cdot2^m$. The detailed proof is provided in Appendix~\ref{sec.appendix.seq-models_globalVS-local.part3}.
\end{proof}
\subsection{Complexity Results}\label{sec:lvs_linvcd}
We now connect the teaching complexity of the sequential models with the VC dimension in the following theorem.
\begin{theorem}\label{thm:main:seq-models_vs_VCD}
$\SigmaLocalTD_{\Instances,\Hypotheses,\hinit} = O(\VCD(\Hypotheses,\Instances)^2)$, and $\SigmaLvsTD_{\Instances,\Hypotheses,\hinit} = O(\VCD(\Hypotheses,\Instances))$.
\end{theorem}


To establish the proof, we first introduce an important definition (\defref{def:compactds}) and two lemmas (\lemref{lem:main:div_concepts} and \lemref{lem:app:vs_dependent_TD}).

\begin{definitionNew}[Compact-Distinguishable Set]\label{def:compactds}
Fix $\hypotheses\subseteq \Hypotheses$ and $\instances \subseteq \Instances$, where $\instances = \{\instance_1, ..., \instance_n\}$. Let $\hypotheses_{|\instances} = \{(\hypothesis(\instance_1), ..., \hypothesis(\instance_n)) \ | \ \forall \hypothesis \in \hypotheses\}$ denote all possible patterns of $\hypotheses$ on $\instances$. 
Then, we say that $\instances$ is \emph{compact-distinguishable} on $\hypotheses$, if $|\hypotheses_{|\instances}| = |\hypotheses|$ and $\forall \instances' \subset \instances,~ |\hypotheses_{|\instances'}| < |\hypotheses|$. We will use $\compactDSet{\hypotheses}$ to denote a compact-distinguishable set on $\hypotheses$. 
\end{definitionNew}
In words, one can uniquely identify any hypothesis in $\hypotheses$ with a (sub)set of examples from $\compactDSet{\hypotheses}$ (also see the definition of distinguishing sets in \cite{doliwa2014recursive}). 
Our definition of compact-distinguishable set further implies that there are no ``redundant'' examples in $\compactDSet{\hypotheses}$. It can be shown that a compact-distinguishable set satisfies the following two properties: 
\begin{enumerate}[(P1)]
\item it does not contain any pair of distinct instances $\instance,\instance'$ such that $(\forall \hypothesis \in \hypotheses: \hypothesis(\instance) = \hypothesis(\instance')) \textnormal{ or } (\forall \hypothesis \in \hypotheses: \hypothesis(\instance) \neq \hypothesis(\instance'))$.
\item it does not contain any instance $\instance$ such that $(\forall \hypothesis \in \hypotheses: \hypothesis(\instance) = 1) \textnormal{ or } (\forall \hypothesis \in \hypotheses: \hypothesis(\instance) = 0)$.
\end{enumerate}

\begin{lemma}\label{lem:main:div_concepts}
\looseness-1Consider a subset $\hypotheses \subseteq \Hypotheses$ and any compact-distinguishable set $\compactDSet{\hypotheses} = \{\instance_1, ..., \instance_{|\compactDSet{\hypotheses}|}\}$. Fix any hypothesis $\hypothesis_\hypotheses \in \hypotheses$. 
Let $d = \VCD(\hypotheses, \compactDSet{\hypotheses})$ denote the VC dimension of $\hypotheses$ on $\compactDSet{\hypotheses}$. If $d \geq 1$, we can divide $\hypotheses$ into $m = |\compactDSet{\hypotheses}| + 1$ separate hypothesis classes 
$\{\hclass{1}, ..., \hclass{m}\}$, such that 
\begin{enumerate}[(i)]\denselist
    \item $\forall j \in [m]$, there exists a compact-distinguishable set $\compactDSet{\hclass{j}}$ s.t. $\VCD(\hclass{j}, \compactDSet{\hclass{j}}) \leq d-1$.
    \item $\forall j \in [m-1]$, $\hypotheses^j$ is not empty and $\hclass{j}_{|\{\instance_j\}} = \{( 1-  \hypothesis_\hypotheses(\instance_j))\}$. 
    \item $\hclass{m} = \{\hypothesis_\hypotheses\}$.
\end{enumerate}
\end{lemma}

\lemref{lem:main:div_concepts} suggests that for any 
$\Hypotheses, \Instances$, one can partition the hypothesis class $\Hypotheses$ into $m \leq |\Instances|+1$ subsets with lower VC dimension with respect to some compact-distinguishable set.\footnote{When $\VCD(\hypotheses,\compactDSet{\hclass{}})=0$, this implies $|\hypotheses|=1$.} The main idea of the lemma is similar to the reduction of a concept class w.r.t. some instance $x$ to lower \VCD as done in Theorem 9 of \cite{floyd1995sample}. The key distinction of \lemref{lem:main:div_concepts} is that we consider compact-distinguishable sets for this partitioning, which in turn ensures the uniqueness of the version spaces associated with these partitions (see proof of Theorem~\ref{thm:main:seq-models_vs_VCD}).
Another key novelty in our proof of Theorem~\ref{thm:main:seq-models_vs_VCD} is to recursively apply the reduction step from the lemma.


To prove the lemma, we provide a constructive procedure to partition the hypothesis class, and show that the resulting partitions have reduced VC dimensions on some compact-distinguishable set. We highlight the procedure for constructing the partitions in \algref{alg:lemma_linVCD} (\lnref{alg:ln:lemma_linvcd_start}-- \lnref{alg:ln:lemma_linvcd_end}). In \figref{fig:warmuth_lemma_run}, we provide an illustrative example  for creating such partitions for the Warmuth hypothesis class from \tableref{tab:warmuth_class_sequence_pref}. We sketch the proof of \lemref{lem:main:div_concepts} below; for a detailed proof, we refer the reader to \cite{mansouri2019preference}.
\begin{proof}[Proof Sketch of \lemref{lem:main:div_concepts}]
\looseness -1 Let us define 
$\hypotheses_{\instance} = \{\hypothesis \in \hypotheses: {\hypothesis \triangle \instance}_{|\compactDSet{{\hypotheses}}} \in \hypotheses_{|\compactDSet{{\hypotheses}}}\}$.
Here, $\hypothesis \triangle \instance$ denotes the hypothesis that only differs with $\hypothesis$ on the label of $\instance$, and $\hypothesis_{|\compactDSet{{\hypotheses}}}$ denotes the patterns of $\hypothesis$ on $\compactDSet{{\hypotheses}}$. Fix a reference hypothesis $\hypothesis_\hypotheses$. For all $j \in [m-1]$, let $\clabel_j = 1 -  \hypothesis_\hypotheses(\instance_j)$ be the opposite label of $\instance_j \in \compactDSet{\hypotheses}$ as provided by $\hypothesis_\hypotheses$.  As shown in \lnref{alg:linvcd:partition} of \algref{alg:lemma_linVCD}, we consider the set $\hclass{1} := \hypotheses^{\clabel_1}_{\instance_1} = \{\hypothesis \in \hypotheses_{\instance_1}: \hypothesis(\instance_1) = \clabel_1\}$ as the first partition. In the detailed proof, we show that $|\hclass{1}| > 0$.


Next, we show that the statement $\VCD(\hclass{1}, \compactDSet{\hclass{}}\setminus\{\instance_1\}) \leq d-1$ holds.
When $d>1$, we prove the statement as follows:
\begin{equation*}
     \VCD(\hclass{1}, \compactDSet{\hypotheses} \setminus\{\instance_1\}) \leq \VCD(\hypotheses^{\clabel_1}_{\instance_1}, \compactDSet{\hypotheses}) = \VCD(\hypotheses_{\instance_1}, \compactDSet{\hypotheses}) - 1 \leq \VCD(\hypotheses, \compactDSet{\hypotheses}) - 1 \leq d - 1.
\end{equation*}
%
In the detailed proof, we prove the statement for $d=1$, and further show that there exists a compact-distinguishable set $\compactDSet{\hclass{1}} \subseteq \compactDSet{\hypotheses} \setminus\{\instance_1\}$ for the first partition $\hclass{1}$.
%
Then, we conclude that the first partition $\hclass{1}$ has $\VCD(\hclass{1}, \compactDSet{\hclass{1}}) \leq d - 1$.

Next, we remove the first partition $\hclass{1}$ from $\hypotheses$, and continue to create the above mentioned partitions on $\hypotheses_{\text{rest}} = \hypotheses \setminus \hclass{1}$ and $\instances_{\text{rest}} = \compactDSet{\hypotheses} \setminus \{\instance_1\}$. 
Then, we show that $\instances_{\text{rest}}$ is a compact-distinguishable set on $\hypotheses_{\text{rest}}$. 
Therefore, we can repeat the above procedure (\lnref{alg:ln:lemma_linvcd_start}-- \lnref{alg:ln:lemma_linvcd_end}, \algref{alg:linvcd}) to create the subsequent partitions. This process continues until the size of $\instances_{\text{rest}}$ reduces to $1$, i.e. $\instances_{\text{rest}} = \{\instance_{m-1}\}$. Until then, we obtain partitions $\{\hclass{1}, ..., \hclass{m-2}\}$. By construction, $\hclass{j}$ satisfy properties (i) and (ii) for all $j \in [m-2]$.

\looseness -1 It remains to show that $\hclass{m-1}$ and $\hclass{m}$ also satisfy the properties in \lemref{lem:main:div_concepts}. Since $\instances_{\text{rest}} = \{\instance_{m-1}\}$ before we start iteration $m-1$, and $\instances_{\text{rest}}$ is a compact-distinguishable set for 
$\hypotheses_{\text{rest}}$, there must exist exactly two hypotheses in $\hypotheses_{\text{rest}}$, and therefore $|\hclass{m-1}|, |\hclass{m}|=1$. This implies that $\VCD(\hclass{m-1}, \compactDSet{\hclass{m-1}}) = \VCD(\hclass{m}, \compactDSet{\hclass{m}}) = 0$. Furthermore, $\forall j \in [m-1]$ and $\hypothesis \in \hclass{j}$, we have  $\hypothesis_\hypotheses(\instance_j) \neq \hypothesis(\instance_j)$. This indicates $\hypothesis_\hypotheses \in \hypotheses_m$, and hence $\hypotheses_m = \{\hypothesis_\hypotheses\}$ which completes the proof.
\end{proof}

\newcommand*{\tikzmk}[1]{\tikz[remember picture,overlay,] \node (#1) {};\ignorespaces}
\newcommand{\boxit}[1]{\tikz[remember picture,overlay]{\node[yshift=3pt,fill=#1,opacity=.25,fit={($(A)-(.08\linewidth,-.2\baselineskip)$)($(B)+(.92\linewidth,.8\baselineskip)$)}] {};}\ignorespaces}

\makeatletter
    \renewcommand{\ALG@name}{Algorithm}
    \makeatother
    \begin{algorithm}[t]
      \caption{Recursive procedure for constructing $\pref_\lvs$, s.t. $\TD_{\Instances,\Hypotheses,\hinit}(\pref_\lvs) \leq \VCD(\Hypotheses, \Instances)$}\label{alg:linvcd}
          \hspace*{\algorithmicindent} \textbf{Input:} $\Instances$, $\Hypotheses$, $\hinit$
        \begin{algorithmic}[1]
            \State Let $I : \Hypotheses \rightarrow \{1, \dots, |\Hypotheses|\}$ be any bijective mapping 
            \State For all $\hypothesis' \in \Hypotheses$, $\hypotheses \subseteq \Hypotheses$, $\hypothesis \in \Hypotheses$, initialize 
            \[\pref_\lvs(\hypothesis'; \hypotheses, \hypothesis) \leftarrow
            \vspace{-4mm}
            \begin{cases}
            0 & \text{if}~\hypothesis' = \hypothesis\\
            |\Hypotheses| + 1 & \text{o.w.}
            \end{cases}
            \]\label{alg:ln:init_pref}
            \State $\textsc{SetPreference}(\Hypotheses, \Hypotheses, \Instances, \hinit)$
            \Function{SetPreference}{$V, {\hypotheses}, {\instances}, \hypothesis$}
            \State Create compact-distinguishable set $\compactDSet{{\hypotheses}} \subseteq {\instances}$
            \State $\hypotheses_{\text{rest}} := {\hypotheses},  \instances_{\text{rest}} := \compactDSet{{\hypotheses}}$
            \For{\tikzmk{A} $\instance \in \compactDSet{{\hypotheses}}$
            }
            \label{alg:ln:lemma_linvcd_start}
                \State $\clabel = 1 - \hypothesis(\instance)$
                %
                \State ${\hypotheses^{\clabel}_{\instance} \leftarrow \{\hypothesis' \in {\hypotheses_{\text{rest}}} : {\hypothesis' \triangle \instance}_{|{\instances_{\text{rest}}}}  \in {\hypotheses_{\text{rest}}}_{|{\instances_{\text{rest}}}}, \hypothesis'(\instance) = \clabel\}}$\label{alg:linvcd:partition}
                \State ${{\hypotheses_{\text{rest}}} \leftarrow {\hypotheses_{\text{rest}}} \setminus \hypotheses^{\clabel}_{\instance} }$,
                \label{alg:ln:lemma_linvcd_end}
                ${\instances_{\text{rest}}} \leftarrow {\instances_{\text{rest}}} \setminus\{x\}$
                \State 
                \tikzmk{B}\boxit{black!30} 
                $V_{\text{next}} \leftarrow V \cap \Hypotheses(\{(\instance,\clabel)\})$ 
                \State \textbf{for} {$\hypothesis' \in \hypotheses^{\clabel}_{\instance}$}
                    \textbf{ do } $\pref_\lvs(\hypothesis'; V_{\text{next}}, \hypothesis) \leftarrow \indexof{\hypothesis'}$\label{alg:ln:set_pref_value}
                \State $\hnext \leftarrow \argmin_{{\hypothesis'} \in \hypotheses^{\clabel}_{\instance}} \indexof{\hypothesis'}$ \label{alg:ln:assign_h_next}
                \State
                $\textsc{SetPreference}(V_{\text{next}}, \hypotheses^{\clabel}_{\instance}, \compactDSet{{\hypotheses}} \setminus\{x\}, \hnext)$\label{alg:ln:recursion}
            \EndFor
            \EndFunction
    	\end{algorithmic}\label{alg:lemma_linVCD}
    \end{algorithm}
    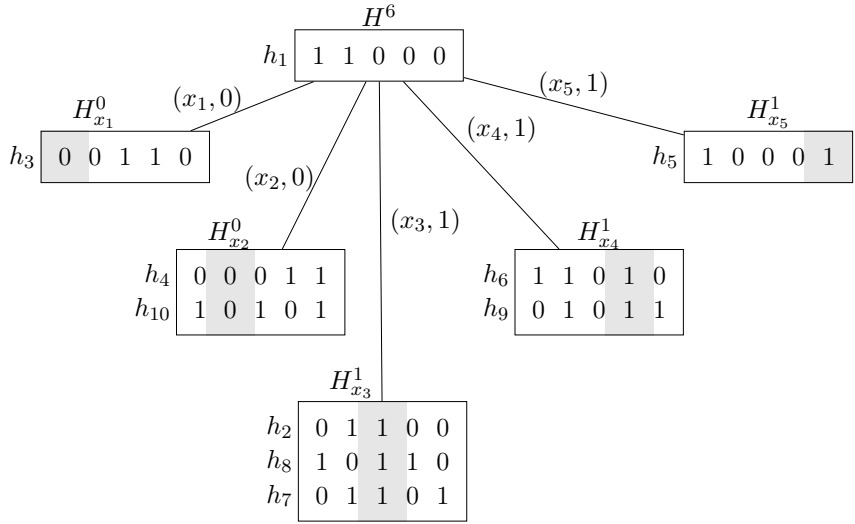
\begin{figure}[!ht]
    \centering
    \scalebox{0.9}{
    \hbox{\hspace{-1cm}
    \begin{tikzpicture}
  [every matrix/.append style={ampersand replacement=\&,matrix of nodes},
  level 1/.style={sibling distance=3.5cm},
  nodes={anchor=west}]
  \node [matrix,draw] at (-1,0) (m0) {1 \& 1 \& 0 \& 0 \& 0 \\}
  child {node[matrix,draw] (m1) at (2,0) {0 \& 0 \& 1 \& 1 \& 0 \\} edge from parent node[left,yshift=1mm] {$(x_1, 0)$}}
  child {node[matrix,draw] (m2) at (0.5,-2)
    {0 \& 0 \& 0 \& 1 \& 1 \\
      1 \& 0 \& 1 \& 0 \& 1 \\
    }
    edge from parent node[left,yshift=-2mm] {$(x_2, 0)$}
  }
  child {node[matrix,draw] (m3) at (-1.2,-4.5)
    {0 \& 1 \& 1 \& 0 \& 0 \\
      1 \& 0 \& 1 \& 1 \& 0 \\
      0 \& 1 \& 1 \& 0 \& 1 \\
    }
    edge from parent node[right,yshift=3mm] {$(x_3, 1)$};
  }
  child {node[matrix,draw] (m4) at (-1.5,-2)
    {1 \& 1 \& 0 \& 1 \& 0 \\
      0 \& 1 \& 0 \& 1 \& 1 \\}
    edge from parent node[above,yshift=2mm,xshift=3mm] {$(x_4, 1)$}
  }
  child {node[matrix,draw] (m5) at (-2.5, 0) {1 \& 0 \& 0 \& 0 \& 1 \\} edge from parent node[above] {$(x_5, 1)$}};

  \node[left,xshift=-3mm] at (m0-1-1) {$h_1$};
  \node[left,xshift=-3mm] at (m1-1-1) {$h_3$};
  \node[left,xshift=-3mm] at (m2-1-1) {$h_4$};
  \node[left,xshift=-3mm] at (m2-2-1) {$h_{10}$};
  \node[left,xshift=-3mm] at (m3-1-1) {$h_2$};
  \node[left,xshift=-3mm] at (m3-2-1) {$h_8$};
  \node[left,xshift=-3mm] at (m3-3-1) {$h_7$};
  \node[left,xshift=-3mm] at (m4-1-1) {$h_6$};
  \node[left,xshift=-3mm] at (m4-2-1) {$h_9$};
  \node[left,xshift=-3mm] at (m5-1-1) {$h_5$};

  \node[above,yshift=3mm] at (m0) {$H^{6}$};
  \node[above left,yshift=3mm] at (m1) {$H_{x_1}^0$};
  \node[above left,yshift=5mm] at (m2) {$H_{x_2}^0$};
  \node[above left,yshift=8mm] at (m3) {$H_{x_3}^1$};
  \node[above,yshift=5mm] at (m4) {$H_{x_4}^1$};
  \node[above,yshift=3mm] at (m5) {$H_{x_5}^1$};
  
  \scoped[on background layer]
  {
    \node[fill=black!10, fit=(m1-1-1)(m1-1-1) ]   {};
    \node[fill=black!10, fit=(m2-1-2)(m2-2-2) ]   {};
    \node[fill=black!10, fit=(m3-1-3)(m3-3-3) ]   {};
    \node[fill=black!10, fit=(m4-1-4)(m4-2-4) ]   {};
    \node[fill=black!10, fit=(m5-1-5)(m5-1-5) ]   {};
  }
\end{tikzpicture}}}
\caption{Illustration of \lemref{lem:main:div_concepts} on the Warmuth class. The grouped hypotheses in the leaf clusters correspond to the sets $\hypotheses^{\clabel}_{\instance}$ created in \lnref{alg:linvcd:partition} of \algref{alg:linvcd}.}\label{fig:warmuth_lemma_run}
\end{figure}

\begin{figure}[!h]
    \centering
    \scalebox{0.70}{
    \begin{tikzpicture}[
  every matrix/.append style={ampersand replacement=\&,matrix of nodes},
  subordinate/.style={%
    grow=down,
    xshift=-3.2em, 
    text centered, text width=12em,
    edge from parent path={(\tikzparentnode.205) |- (\tikzchildnode.west)}
  },
  level1/.style ={level distance=4em,anchor=north},
  level2/.style ={level distance=8em,anchor=north},
  level 1/.style={edge from parent fork down,sibling distance=10em,level distance=5em},
  level 2/.style={edge from parent fork down,sibling distance=10em},  
  nodes={anchor=west}]
  \node [matrix,draw] (m0) {1 \& 1 \& 0 \& 0 \& 0 \\}
  child[level 1] {node[matrix,draw] (m1) 
    {0 \& 0 \& 1 \& 1 \& 0 \\} edge from parent node[above] {$(x_1, 0)$}}
  child[level 1] {node[matrix,draw] (m2) 
    {0 \& 0 \& 0 \& 1 \& 1 \\
    }
    child[level 2] {node[matrix,draw] (m6) at (-3,0)
      {
        1 \& 0 \& 1 \& 0 \& 1 \\
      }
      edge from parent node[above] {$(x_3, 1)$}
    }
    edge from parent node[above] {$(x_2, 0)$}
  }
  child[level 1] {node[matrix,draw] (m3) 
    {
      0 \& 1 \& 1 \& 0 \& 0 \\
    }
    child[level 2]{node[matrix,draw] (m7) at (-1,0)
      {1 \& 0 \& 1 \& 1 \& 0\\}
      edge from parent node[above] {$(x_4, 1)$}
    }
    child[level 2]{node[matrix,draw] (m8) at (-1.5,0)
      {
        0 \& 1 \& 1 \& 0 \& 1 \\
      }
      edge from parent node[above] {$(x_5, 1)$}
    }
    edge from parent node[above] {$(x_3, 1)$}
  }
  child[level 1] {node[matrix,draw] (m4) 
    {1 \& 1 \& 0 \& 1 \& 0 \\}
    child[level 2] {node[matrix,draw] (m9) 
      { 0 \& 1 \& 0 \& 1 \& 1 \\}
      edge from parent node[above] {$(x_5, 1)$}
    }
    edge from parent node[above] {$(x_4, 1)$}
  }
  child[level 1] {node[matrix,draw] (m5) 
    {1 \& 0 \& 0 \& 0 \& 1 \\}
    edge from parent node[above] {$(x_5, 1)$}
  };
  \node[left,xshift=-3mm] at (m0-1-1) {$h_1$};
  \node[left,xshift=-3mm] at (m1-1-1) {$h_3$};
  \node[left,xshift=-3mm] at (m2-1-1) {$h_4$};
  \node[left,xshift=-3mm] at (m6-1-1) {$h_{10}$};
  \node[left,xshift=-3mm] at (m3-1-1) {$h_2$};
  \node[left,xshift=-3mm] at (m7-1-1) {$h_8$};
  \node[left,xshift=-3mm] at (m8-1-1) {$h_7$};
  \node[left,xshift=-3mm] at (m4-1-1) {$h_6$};
  \node[left,xshift=-3mm] at (m9-1-1) {$h_9$};
  \node[left,xshift=-3mm] at (m5-1-1) {$h_5$};
  %
  %
\end{tikzpicture}}
    \caption{Illustration of \thmref{thm:main:seq-models_vs_VCD} proof -- constructing a  $\sigma_\lvs \in \SigmaLvs$ for the Warmuth class.}
    \vspace{-4mm}
    \label{fig:warmuth-lvs-theorem-construction}
\end{figure}
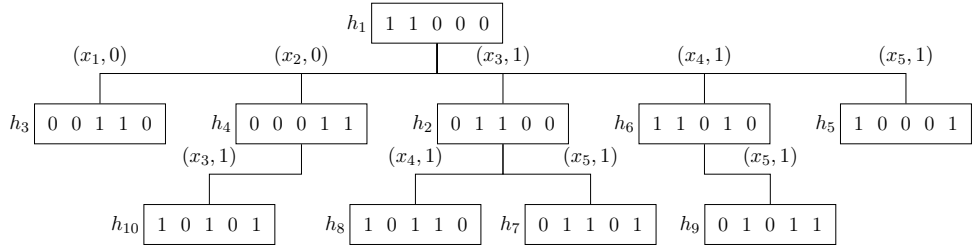

\looseness-1Next, we show that every teaching
example $(\instance_j, \clabel_j)$, where $\instance_j\in \compactDSet{\hypotheses}$ and $\clabel_j=1-\hypothesis(x_j)$ for some fixed $h$, corresponds to a unique version space $V^j$. We will later use this fact in the proof of \thmref{thm:main:seq-models_vs_VCD}. As a more rigorous statement of this fact, we establish the following lemma.
\begin{lemma}\label{lem:app:vs_dependent_TD}
 
 Fix $\hypotheses \subseteq \Hypotheses$, and let $\compactDSet{\hypotheses} \subseteq \Instances$ be a compact-distinguishable set on $\hypotheses$. For any $\instance, \instance' \in \compactDSet{\hypotheses}$ and $\clabel, \clabel' \in \{0,1\}$ such that $(\instance, \clabel) \neq (\instance', \clabel')$, the resulting version spaces $\{\hypothesis \in \hypotheses: \hypothesis(\instance) = \clabel\}$ and $\{\hypothesis \in \hypotheses: \hypothesis(\instance') = \clabel'\}$ are different.
 
\end{lemma}


\begin{proof}[Proof of \lemref{lem:app:vs_dependent_TD}]
\looseness-1Denote $A = \{\hypothesis \in \hypotheses: \hypothesis(\instance) = \clabel\}$ and $B = \{\hypothesis \in \hypotheses: \hypothesis(\instance') = \clabel'\}$. We consider the following two cases: (i) $\clabel = \clabel'$ and (ii) $\clabel \neq \clabel'$. For the case where $\clabel = \clabel'$, if  $A  = B$, this would violate the first condition of the property (P1) of compact-distinguishable sets as stated after Definition~\ref{def:compactds} (i.e., there does not exist distinct $\instance, \instance'$ s.t. $\forall \hypothesis \in \hypotheses, \hypothesis(\instance) = \hypothesis(\instance')$). For the case where $\clabel \neq \clabel'$, if $A  = B$, this would violate the second condition of  (P1) (i.e., there does not exist distinct $\instance, \instance'$ s.t. $\forall \hypothesis \in \hypotheses, \hypothesis(\instance) \neq \hypothesis(\instance')$).  Hence it completes the proof.
\end{proof}
Now we are ready to prove \thmref{thm:main:seq-models_vs_VCD}. As part of the proof, we provide a recursive procedure for constructing a $\sigma_\lvs \in \SigmaLvs$ achieving $\TD_{\Instances,\Hypotheses,\hinit}(\pref_\lvs) = \bigO{\VCD(\Hypotheses,\Instances)}$.

\begin{proof}[Proof of \thmref{thm:main:seq-models_vs_VCD}]
In a nutshell, the proof consists of three steps: (i) initialization of $\pref_\lvs$, (ii) setting the preferences by recursively invoking the constructive procedure for \lemref{lem:main:div_concepts}, and (iii) showing that there exists a teaching sequence of length up to $\VCD(\Hypotheses,\Instances)$ for any target hypothesis $\hstar$. 
We summarize the recursive procedure in \algref{alg:lemma_linVCD}. In \figref{fig:warmuth-lvs-theorem-construction}, we illustrate the recursive construction of a $\pref_\lvs \in \SigmaLvs$ for the Warmuth class.

\textit{\underline{Step (i).}}
To begin with, we initialize $\pref_\lvs$ with default values which induce high $\sigma$ values (i.e., low preference), except for $\sigma(\hypothesis';\hypotheses,\hypothesis)=0$ when $\hypothesis'=\hypothesis$ (\lnref{alg:ln:init_pref} of \algref{alg:lemma_linVCD}). The self-preference guarantees that $\sigma_\lvs$ is collusion-free as per Definition~\ref{def:seq-col-free}. 

\textit{\underline{Step (ii).}}
The recursion begins at the top level with $\hypotheses = \Hypotheses$, current version space $V = \Hypotheses$, and current hypothesis $\hypothesis=\hinit$. \lemref{lem:main:div_concepts} suggests that we can partition $\hypotheses$ into $m=|\compactDSet{\hypotheses}|+1$ groups $\{\hclass{1}, ..., \hclass{m}\}$, where for all $j \in [m]$, there exists a compact-distinguishable set $\compactDSet{\hclass{j}}$ that satisfies the properties in \lemref{lem:main:div_concepts}.


Now consider the hypothesis $\hypothesis := \hinit$. 
We show that for $j \in [m - 1]$, every $(\instance_j, \clabel_j)$, where $\instance_j\in \compactDSet{\hypotheses}$ and $\clabel_j=1-\hypothesis(x_j)$, corresponds to a unique version space $V^j:= \{h \in V : h(\instance_j) =  \clabel_j\}$. To prove this statement, we consider $R^j := {V}^j \cap \hypotheses = \{\hypothesis \in \hypotheses : \hypothesis(\instance_j) = \clabel_j\}$.
According to Lemma~\ref{lem:app:vs_dependent_TD}, we know that none of $R^j$ for $j \in [m-1]$ are equal. This indicates that none of $V^j$ for $j \in [m-1]$ are equal.  

We then set the values of the preference function $\pref_\lvs(\cdot; V^j, h)$ for all $j \in [m - 1]$ and $\clabel_j = 1 - \hypothesis(\instance_j)$ (\lnref{alg:ln:set_pref_value}). Upon receiving $(\instance_j, \clabel_j)$, the learner will be steered to the next ``search space'' $\hclass{j}$, with version space $V^j$. By \lemref{lem:main:div_concepts} we have $\VCD(\hclass{j}, \compactDSet{\hclass{j}})\leq \VCD(\hypotheses, \compactDSet{\hypotheses}) - 1$.

We will build the preference function $\pref_\lvs$ recursively $m - 1$ times for each 
$( V^j, \hclass{j}, \compactDSet{\hclass{j}}, \hypothesis_{\text{next}})$, where $\hypothesis_{\text{next}}$ corresponds to the unique hypothesis identified by function $I$ (\lnref{alg:ln:assign_h_next}--\lnref{alg:ln:recursion}).
At each level of recursion, \VCD reduces by 1. We stop the recursion when $\VCD(\hclass{j}; \compactDSet{\hclass{j}})=0$, 
which corresponds to the scenario $|\hclass{j}|=1$. 

\textit{\underline{Step (iii).}}
Given the preference function constructed in \algref{alg:linvcd}, we can build up the set of teaching examples recursively. 
Consider the beginning of the teaching process, where the learner's current hypothesis is $\hinit$ and version space is $\Hypotheses$, and the goal of the teacher is to teach $\hstar$. Consider the first level of the recursion in \algref{alg:lemma_linVCD}, where we divide $\Hypotheses$ into $m=|\compactDSet{\Hypotheses}|+1$ groups $\{\hclass{1}, ..., \hclass{m}\}$. Let us consider the case where $\hstar \in \hclass{j^\star}$ with $j^\star \in [m-1]$. 
The teacher provides an example given by $(\instance=\instance_{j^\star}, \clabel=\hstar(\instance_{j^\star}))$. After receiving the teaching example, the resulting partition $\hclass{j^\star}$ will stay in the version space; meanwhile, $\hinit$ will be removed from the version space. The new version space will be $V^{j^\star}$. The learner's new hypothesis induced by the preference function is given by $\hypothesis_{\text{next}} \in \hclass{j^\star}$. By repeating this teaching process for a maximum of $\VCD(\Hypotheses,\Instances)$ steps, the learner reaches a partition of size 1 (see \emph{{Step (ii)}} for details). At this step $\hstar$ must be the only hypothesis left in the search space. Therefore, $\hypothesis_{\text{next}} = \hstar$, and the learner has reached $\hstar$. 
%
%
\end{proof}


\looseness-1\paragraph{Remark.} The recursive procedure in \algref{alg:lemma_linVCD} creates a preference function $\pref_\lvs\in\SigmaLvs$ that has teaching complexity at most $\VCD(\Hypotheses,\Instances)$. It is interesting to note that the resulting preference function $\pref_\lvs \in \SigmaWinStayLoseShift \subset \SigmaLvs$ (cf. \secref{sec:complexity:cf}), i.e., it has the characteristic of ``win-stay, loose shift'' \cite{bonawitz2014win,chen2018understanding}.
For some problems, one can achieve lower teaching complexity for a $\sigma \in \SigmaLvs$ which does not have this characteristic. 
For the Warmuth hypothesis class, the preference function $\sigma_{\lvs}$ we provided in \tableref{tab:warmuth_class_sequence_pref} has teaching complexity $1$, while the preference function we constructed in \figref{fig:warmuth-lvs-theorem-construction} has teaching complexity $2$.



\section{Properties of the Teaching Complexity Parameter $\Sigmatd{}$}\label{sec:newresults}
In this section, we analyze several properties of the teaching complexity parameter $\Sigmatd{}$ associated with different families of the preference functions. More concretely, we establish a lower bound on  the teaching parameter $\Sigmatd{}$ w.r.t. \VCD for specific hypotheses class, discuss the additivity/sub-additivity property of $\Sigmatd{}$ over disjoint domains, and compare the sizes of the different families of preference functions.


\subsection{$\Sigmatd{}$ is Not a Constant}
One question of particular interest is showing that the teaching parameter $\Sigmatd{}$ is not upper bounded by any constant independent of the hypothesis class, which would suggest a strong collusion in our model. In the following, we show that for certain hypothesis classes, $\Sigmatd{}$ is lower bounded by a function of $\VCD$, as proved in the lemma below.



\begin{lemma}
Consider the powerset hypothesis class with $\Hypotheses=\{0,1\}^d$ (this class has $\VCD=d$). Then, for any family of collusion-free preference functions $\Sigma \subseteq \Sigma_{\CF}$, $\Sigmatd{}$ is lower bounded by $\bigOmega{\frac{d}{\log d}}$ for this hypothesis class.
\end{lemma}


\begin{proof}
We will use the fact that for any collusion-free preference function $\sigma \in \Sigma$,  the teaching sequences of two distinct hypotheses cannot be exactly the same. As $\Hypotheses$, $\Instances$ denote the power set of size $d \geq 1$, we know that $|\Instances| = d$, $|\Hypotheses| = 2^d$, and $\VCD(\Hypotheses, \Instances) = d$. Also for any $\Sigma$, let us denote $k := \Sigmatd{\Instances, \Hypotheses, \hypothesis_0}$. We will denote $N(k)$ to be the number of teaching sequences of size less than or equal to $k$. Since $d \geq 1$, we have that
\begin{equation*}
    N(k) \leq \sum_{i = 0}^k (2|\Instances|)^i = \sum_{i = 0}^k (2d)^i < (2d)^{k+1}.
\end{equation*}

We note that the total number of unique teaching sequences of size less than or equal to $k$ must be greater than or equal to $|\Hypotheses|$, i.e., we require $N(k) \geq |\Hypotheses| = 2^d$. Therefore, we will conclude that $2^d < (2d)^{k+1}$. This in turn requires that $k$ is $\Omega\big({\frac{d}{\log d}}\big)$.
\end{proof}

In summary, the above lemma shows the existence of hypothesis classes such that $\Sigmatd{\Instances, \Hypotheses, \hypothesis_0}$ is $\Omega\big({\frac{\VCD(\Hypotheses, \Instances)}{\log \VCD(\Hypotheses, \Instances)}}\big)$.


\subsection{Additive and Sub-additive Properties of $\Sigmatd{}$}\label{sec:newresults:sub-add}
In this section, we explore whether the teaching complexity parameter $\Sigmatd{}$ is additive or sub-additive over \emph{disjoint unions of hypothesis classes} \cite{DBLP:journals/jmlr/DoliwaFSZ14,pmlr-v98-kirkpatrick19a}.   These properties have been studied for the existing complexity measures including $\worstcase\textnormal{-}\TD$, $\RTD$, $\NCTD$, and $\VCD$~\cite{DBLP:journals/jmlr/DoliwaFSZ14,pmlr-v98-kirkpatrick19a}. For instance, \cite{DBLP:journals/jmlr/DoliwaFSZ14} leverages the additivity property of $\RTD$ and $\VCD$ to show that the gap between these two complexity measures can be made arbitrarily large by iteratively constructing larger hypothesis classes from the Warmuth hypothesis class. Next, we will formally introduce the notion of additivity/sub-additivity over disjoint unions of hypothesis classes, and then study it for the complexity measure $\Sigmatd{}$ over different families $\Sigma$. These definitions are inspired by existing work, in particular, we refer the reader to Lemma 16 of \cite{DBLP:journals/jmlr/DoliwaFSZ14}, and Section 5 of \cite{pmlr-v98-kirkpatrick19a}. 



\begin{definitionNew}[Disjoint union of hypothesis classes]
Consider two hypothesis classes $\Hypotheses^a$ and $\Hypotheses^b$ over two disjoint instance spaces $\Instances^a$ and $\Instances^b$ respectively, i.e.,  $\Instances^a \cap \Instances^b = \varnothing$. We define a disjoint union of these hypothesis classes as $\Hypotheses^a \uplus \Hypotheses^b = \{\hypothesis^a \uplus \hypothesis^b \mid \hypothesis^a \in \Hypotheses^a, \hypothesis^b \in \Hypotheses^b\}$ where $\hypothesis := \hypothesis^a \uplus \hypothesis^b$ is a function mapping $\Instances^a \cup \Instances^b$ to $\Clabels$ such that
\begin{align*}
\hypothesis(\instance) = 
\begin{cases}
1 & \text{~if~} \hypothesis^a(\instance) = 1, \text{~when~} \instance \in \Instances^a \\ 
1 & \text{~if~} \hypothesis^b(\instance) = 1, \text{~when~} \instance \in \Instances^b\\
0 & \text{otherwise}
\end{cases}
\end{align*}
\label{def:disjointunion}
\end{definitionNew}

\vspace{2mm}
\begin{definitionNew}[Additive and sub-additive property]
Consider a family of preference functions $\Sigma$ to be one of the families studied, i.e.,  $\Sigma \in \{\SigmaConst, \SigmaGlobal, \SigmaGvs, \SigmaLocal, \SigmaWinStayLoseShift, \SigmaLvs\}$. Then, $\Sigmatd{}$ is additive/sub-additive over the operator $\uplus$ as defined in Definition~\ref{def:disjointunion}, if for any two hypothesis classes $\Hypotheses^a$ and $\Hypotheses^b$ over two disjoint instance spaces $\Instances^a$ and $\Instances^b$ respectively, and any $\hypothesis_0^a \in \Hypotheses^a,  \hypothesis_0^b \in \Hypotheses^b$, the following holds:
\begin{align*}
    (\textbf{Additivity}) \ \ &\Sigmatd{\Instances^a \cup \Instances^b, \Hypotheses^a \uplus \Hypotheses^b, \hypothesis_0^a \uplus \hypothesis_0^b} = \Sigmatd{\Instances^a, \Hypotheses^a,\hypothesis_0^a} + \Sigmatd{\Instances^b, \Hypotheses^b,\hypothesis_0^b}\\
    (\textbf{Sub-additivity}) \ \ &\Sigmatd{\Instances^a \cup \Instances^b, \Hypotheses^a \uplus \Hypotheses^b, \hypothesis_0^a \uplus \hypothesis_0^b}  \leq  \Sigmatd{\Instances^a, \Hypotheses^a,\hypothesis_0^a} + \Sigmatd{\Instances^b, \Hypotheses^b,\hypothesis_0^b}  
\end{align*}
\label{def:additive-subadditive}
\end{definitionNew}

We first establish the additive/sub-additive properties of batch preference families, namely,  $\Sigma \in \{\SigmaConst, \SigmaGlobal, \SigmaGvs\}$. In  Lemma 16 of \cite{DBLP:journals/jmlr/DoliwaFSZ14}, it is shown that $\RTD$ is additive. Similarly, it can be shown that $\worstcase\textnormal{-}\TD$ is additive: This follows from \cite{goldman1995complexity} where it is clear that the optimal teaching set for the worst-case model can be obtained as a solution to a set cover problem, and a disjoint union of hypothesis classes leads to two disjoint set cover problems. In the recent work \cite{pmlr-v98-kirkpatrick19a}, it has been proven that $\NCTD$ is sub-additive; also, it has been shown that for certain hypothesis classes $\NCTD$ acts \emph{strictly} sub-additive, i.e., the $\leq$ relation in Definition~\ref{def:additive-subadditive} holds with $<$. The equivalence results in the Theorem~\ref{theorem:equivelence-results} directly establish that $\SigmaConstTD$ is additive, $\SigmaGlobalTD$ is additive, and $\SigmaGvsTD$ is sub-additive.

Next, we study these properties for local families of preference functions that depend on the learner's current hypothesis. In particular, the lemma below establishes the sub-additive property for an important family of local preference functions $\SigmaWinStayLoseShift$; furthermore, this property holds strictly.

\vspace{2mm}
\begin{lemma}\label{lem:additional:subadditive}
Consider the family of preference functions $\Sigma := \SigmaWinStayLoseShift$. Then, for any two hypothesis classes $\Hypotheses^a$ and $\Hypotheses^b$ over two disjoint instance spaces  $\Instances^a$ and $\Instances^b$ respectively, and any $\hypothesis_0^a \in \Hypotheses^a,  \hypothesis_0^b \in \Hypotheses^b$, the sub-additive property holds, i.e.,
\begin{align*}
    \Sigmatd{\Instances^a \cup \Instances^b, \Hypotheses^a \uplus \Hypotheses^b, \hypothesis_0^a \uplus \hypothesis_0^b} \leq \Sigmatd{\Instances^a, \Hypotheses^a,\hypothesis_0^a} + \Sigmatd{\Instances^b, \Hypotheses^b,\hypothesis_0^b}\\
\vspace{-8mm}
\end{align*}
Furthermore, the sub-additive property holds strictly, i.e., there exist hypothesis classes where the relation $\leq$ above holds with $<$.
\end{lemma}

\vspace{4mm}
The proof is provided in Appendix~\ref{appendix:newresults:sub-add}. We conjecture that the sub-additive property also holds for the more general family $\SigmaLvs$ and leave the proof as future work.

\subsection{Teaching Complexity $\Sigmatd{}$  w.r.t. the Size of $\Sigma$}
The results in Table~\ref{tab:overview} showcase that the teaching complexity $\Sigmatd{}$ goes down as we consider more powerful families of preference functions.  Here, we discuss this reduction in the teaching complexity from the viewpoint of the size of the family---the larger the set $\Sigma$, the teacher/learner can find a better $\sigma \in \Sigma$ in Eq.~\eqref{eq:sigmatd} achieving a lower teaching complexity. The Venn diagram in \figref{fig:venndiagram_seq} already illustrated the relationship between different families of preference functions $\Sigma$, and here we provide a more quantitative view of this Venn diagram. 

Consider a hypothesis class $\Hypotheses$ over an instance space $\Instances$. Let $m = |\Hypotheses|$ denote the size of the hypothesis class, $N$ denote the number of possible version spaces that can be induced by labeled instances (upper bounded by $2^{m})$, and let $C:\mathbb{Z}_{+} \rightarrow \mathbb{Z}_{+}$ be  a function given by:
\begin{align*}
	C(m) =\sum_{k = 1}^m \ \ \sum_{t \in \mathbb{Z}_{+}:\  t_1 + t_2 + \ldots t_{k} = m} \left(\begin{array}{c} m \\ t_1, t_2, \ldots, t_{k} \end{array}\right).
\end{align*}

\looseness-1Next we discuss the size of different families denoted as $|\Sigma|$ in terms of $m$, $N$,  and $C(m)$. Note that we are not interested in the actual number of possible $\sigma$ functions in the set $\Sigma$---this number is unbounded even for the simplest family $\SigmaConst$ as the preferences are given in terms of real-valued functions. Instead, we measure the size $|\Sigma|$ in terms of the possible number of preference relations that can be induced within a given family. Below, we illustrate how the  size of the families grows as we go from $\SigmaConst$, $\SigmaGlobal$ to $\SigmaGvs$/$\SigmaLocal$, and finally to $\SigmaLvs$:
\begin{itemize}
	\item $\SigmaConst$:  We have $|\SigmaConst|=1$ as all the hypotheses are equally preferred for any $\sigma \in \SigmaConst$.
	\item $\SigmaGlobal$: We have $|\SigmaGlobal| = C(m)$ as the function $C$ defined above computes the number of preference relations that can be induced by a global preference function. 
	\item $\SigmaGvs$: These preference functions depend on the learner’s version space and $|\SigmaGvs|$ grows as $\big(C(m)\big)^{N}$.
	\item $\SigmaLocal$:  These preference functions depend on the learner’s current hypothesis and $|\SigmaLocal|$ grows as $\big(C(m)\big)^{m}$.
	\item $\SigmaLvs$: These preference functions depend on the learner’s current hypothesis and the version space, inducing a powerful family of preference relations. $|\SigmaLvs|$ grows as $\big(C(m)\big)^{m \cdot N}$.
\end{itemize}
	


%
%

\paragraph{Remark on run time complexity.} While run time has not been the focus of this paper, it would be interesting to characterize the presumably increased run time complexity of sequential learners and teachers with complex preference functions. Furthermore, as the size of the families grows, the problem of finding the best preference function $\sigma$ in a given family $\Sigma$ that achieves the minima in Eq.~\eqref{eq:sigmatd} becomes more computationally challenging.

\section{Conclusion and Future Work}\label{sec:discussion}
In this paper, we introduced a general preference-based teaching model, which encompasses a number of previously studied batch and sequential models. In particular, we showed that the classical worst-case teaching model, the recursive/preference-based teaching model, the no-clash teaching model, and the local preference-based teaching model could be viewed as special cases of our model, corresponding to different families of preference functions. We then provided a procedure for constructing preference functions $\sigma$ which induce a novel family of sequential models with teaching complexity $\TD(\sigma)$ linear in the VC dimension: this is in contrast to the best-known complexity result for the batch models, which is quadratic in the VC dimension. We further analyzed several properties of the teaching complexity parameter $\TD(\sigma)$ associated with different families of the preference functions.
One fundamental aspect of modeling teacher-learner interactions is the notion of collusion-free teaching. Collusion-freeness for the batched setting is well established in the research community and \NCTD characterizes the complexity of the strongest collusion-free batch model. In this paper, we are introducing a new notion of collusion-freeness for the sequential setting (\defref{def:seq-col-free}). As discussed at the end of \secref{sec:lvs_linvcd}, a stricter notion is the ``win-stay lose-shift'' condition, which is easier to validate without running the teaching algorithm. In contrast, the condition of \defref{def:seq-col-free} is more involved in terms of validation and is a joint property of the teacher-learner pair. One intriguing question for future work is defining notions of collusion-free teaching in sequential models and understanding their implications on teaching complexity.
\looseness-1 Our framework provides novel tools for reasoning about teaching complexity by constructing preference functions. This opens up an interesting direction of research to tackle important open problems, such as proving whether \NCTD or \RTD is linear in \VCD~\cite{simon2015open,chen2016recursive,DBLP:journals/corr/HuWLW17,pmlr-v98-kirkpatrick19a}. In this paper, we showed that neither of the families $\SigmaGvs$ and $\SigmaLocal$ dominates the other (Theorem~\ref{thm:local-eq-GVS}). As a direction for future work, it would be important to further quantify the complexity of $\SigmaLocal$ family.

\subsubsection*{Acknowledgements}
Yuxin Chen is supported by NSF 2037026, and a C3.ai DTI Research Award 049755. 
Xiaojin Zhu is supported by NSF 1545481, 1561512, 1623605, 1704117, 1836978 and the MADLab AF CoE FA9550-18-1-0166.

\bibliography{main}

\iftoggle{longversion}{
\clearpage
\onecolumn
\appendix 
{\allowdisplaybreaks
\section{Supplementary Materials for \secref{sec:non-seq-family-pref}: Proof of \thmref{theorem:equivelence-results}} \label{sec.appendix.batch-models}

Before we prove our main results for the batch models, we first establish the following results on the non-clashing teaching. The notion of a non-clashing teacher was first introduced by \cite{KKW07-nonclashing}. Our proof is inspired by \cite{pmlr-v98-kirkpatrick19a}, which shows the non-clashing property for collusion-free teacher-learner pair in the batch setting. 
\begin{lemma}\label{lm:seq-vs-bat_coll-free} 
Consider a collusion-free preference function $\sigma \in \SigmaGvs$. Then, a successful teacher $\Teacher$ w.r.t. a learner $L_\sigma$ with preferences $\sigma$ must be non-clashing on $\Hypotheses$. i.e., for any two distinct $\hypothesis, \hypothesis' \in \Hypotheses$ such that $\Teacher(\hypothesis)$ is consistent with $\hypothesis'$, $\Teacher(\hypothesis')$ cannot be consistent with $\hypothesis$.
\end{lemma}
\begin{proof}[Proof of \lemref{lm:seq-vs-bat_coll-free}] 
By definition of the preference function, we have  $\forall \sigma\in \SigmaGvs, \hypothesis'\in \Hypotheses$, $\sigma(\hypothesis'; \Hypotheses(\examples'), \cdot) = g_\sigma(\hypothesis', \Hypotheses(\examples'))$ for some function $g_\sigma$.
We prove the lemma by contradiction. Assume that the teacher mapping $\Teacher$ is not non-clashing. This assumption implies that there exists $\hypothesis \neq \hypothesis' \in \Hypotheses$, where $\examples = \Teacher(\hypothesis)$ and $\examples' = \Teacher(\hypothesis')$ are consistent with both $\hypothesis$ and $\hypothesis'$.

Assume that the last current hypothesis before the teacher provides the last example of $\examples$ is $\hypothesis_1$. Then,
\begin{equation*}
    \{\hypothesis\} = \argmin_{\hypothesis'' \in \Hypotheses(\examples)} \sigma(\hypothesis'';\Hypotheses(\examples),\hypothesis_1) = \argmin_{\hypothesis'' \in \Hypotheses(\examples\cup\examples')} \sigma(\hypothesis''; \Hypotheses(\examples \cup \examples'), \hypothesis) = \argmin_{\hypothesis'' \in \Hypotheses(\examples \cup \examples')} g_\sigma(\hypothesis'', \Hypotheses(\examples \cup \examples'))
\end{equation*}
where the first equality is the definition of a teaching sequence and the second equality is by the definition of collusion-free preference function (\defref{def:seq-col-free}). 
Similarly we have
\begin{equation*}
    \{\hypothesis'\} =  \argmin_{\hypothesis'' \in \Hypotheses(\examples' \cup \examples)} g_\sigma(\hypothesis'', \Hypotheses(\examples' \cup \examples)).
\end{equation*}
Consequently, $\hypothesis = \hypothesis'$, which is a contradiction. This indicates that $\Teacher$ is non-clashing.
\end{proof}

Now we are ready to provide the proof for \thmref{theorem:equivelence-results}. We divide the proof of the \thmref{theorem:equivelence-results} into three parts, each corresponding to the equivalence results for a different preference function family.
\begin{proof}[Proof of \thmref{theorem:equivelence-results}] 
Part 1 (reduction to \worstcase-\TD) and Part 2 (reduction to \RTD) of the proof are included in the main paper.
%
%
%
For Part 3, i.e., to establish the equivalence between $\SigmaGvsTD$ and $\NCTD$, it suffices to show that for any $\Instances,\Hypotheses,\hinit$, the following holds:
\begin{enumerate}[(i)]
	\item $\SigmaGvsTD_{\Instances,\Hypotheses,\hinit} \geq \NCTD(\Hypotheses)$
	\item $\SigmaGvsTD_{\Instances,\Hypotheses,\hinit} \leq \NCTD(\Hypotheses)$
\end{enumerate}

We first prove (i). 
According to \lemref{lm:seq-vs-bat_coll-free}, for any $\pref \in \Sigma_{\gvs}$, a successful teacher $\Teacher$ w.r.t. a learner $L_\sigma$ is non-clashing on $\Hypotheses$. Therefore, we have the following: $$\SigmaGvsTD_{\Instances,\Hypotheses,\hinit} =   \min_{\text{Successful Teacher $\Teacher$}}\max_{\hypothesis \in \Hypotheses}  |\Teacher(\hypothesis)| \geq \min_{\text{Non-clashing Teacher~} \Teacher} \max_{h\in \Hypotheses}|\Teacher(\hypothesis)| = \NCTD(\Hypotheses).$$


We now proceed to prove (ii). Consider any non-clashing teacher mapping $\Teacher$. 
We will prove by showing that there exists a collusion-free $\pref \in \Sigma_\gvs$ such that $\Teacher$ is successful w.r.t. a learner $L_\sigma$ on $\Hypotheses$.
We can construct such as a preference function $\sigma$ as follows. First, we initialize $\sigma(\cdot;\cdot,\cdot) = 1$. Then, for every $\hypothesis \in \Hypotheses$ and every $S$ such that $\Teacher(h) \subseteq S$ and $S$ is consistent with $\hypothesis$, we assign $\sigma(\hypothesis; \Hypotheses(S),\cdot) = 0$.

As shown in the earlier version of this paper~\cite{mansouri2019preference}, the $\sigma$ function constructed above is collusion-free and the teacher mapping $\Teacher$ is successful for the learner $L_\sigma$. Therefore, we conclude that for any non-clashing teacher mapping $T$, we can construct a preference function $\sigma\in \SigmaGvs$ such that $\max_{\hypothesis\in\Hypotheses}|\Teacher(h)| \geq \TD_{\Instances,\Hypotheses,\hinit}(\pref)$. Consequently, $\SigmaGvsTD_{\Instances,\Hypotheses,\hinit} \leq \NCTD(\Hypotheses)$. Combining this result with (i) completes the proof for Part3.
\end{proof}
\section{Supplementary Materials for \secref{sec:non-seq-family-pref}: Proof of \thmref{theorem:equivelence-results-infinite}} \label{sec.appendix.batch-models-infinite}
We extend our results for batch models in \secref{sec:batch-complexity-finite} to infinite domain. We first introduce the necessary notations and definitions here, expanding on the presentation in \secref{sec:batch-complexity-infinite}.


\subsection*{The Teaching Model with Preference Functions}  
Let $\Instances^c$ be an infinite ground set of unlabeled instances and let $\Hypotheses^c$ be an infinite class of hypotheses; each hypothesis $\hypothesis \in \Hypotheses^c$ is a function $\hypothesis : \Instances^c \to \Clabels$. Let $\Examples^c \subseteq \Instances^c \times \Clabels$ be the ground set of labeled examples. Again, for any $\examples \subseteq \Examples^c$, we define $\Hypotheses^c(\examples):= \{\hypothesis \in \Hypotheses^c \mid \forall \example = (\instance_\example, \clabel_\example) \in \examples, \hypothesis(\instance_\example) = \clabel_\example\}$.

The preference functions for the infinite domain are given by $\sigma^c: \Hypotheses^c \times 2^{\Hypotheses^c} \times \Hypotheses^c \rightarrow \reals$. Adapting \eqref{eq.learners-jump} to the infinite domain, the learner picks the next hypothesis based on the current hypothesis $\hypothesis_{t-1}$, version space $\hypotheses_{t}$, and preference function $\sigma_c$:
\begin{align}
\hypothesis_{t} \in \arg \inf_{\hypothesis'\in \hypotheses_{t}} \sigma^c(\hypothesis'; \hypotheses_t, \hypothesis_{t-1}).
\label{eq.learners-jump-infinite}
\end{align}

\subsection*{The Complexity of Teaching with Preference Functions}
Fix preference function $\sigma^c$. For any version space $\hypotheses\subseteq\Hypotheses^c$, the worst-case optimal cost for steering the learner from $\hypothesis$ to $\hstar$ is characterized by

\begin{align*}
  \Val_{\pref^c}(\hypotheses, \hypothesis, \hstar) =
  \begin{cases}
    1, & 
    \exists z, \text{~s.t.~}\Candidate_{\sigma^c}(\hypotheses,\hypothesis,\example) = \{\hstar\}\\
    1 + \inf\limits_{\example} \sup\limits_{\hypothesis'' \in \Candidate_{\sigma^c}(\hypotheses, \hypothesis, \example) 
    } \Val_{\pref^c}(\hypotheses \cap \Hypotheses^c(\{\example\}), \hypothesis'', \hstar) ,  &\text{otherwise}
  \end{cases}
\end{align*}
where 
    $\Candidate_{\sigma^c}(\hypotheses, \hypothesis, \example) =\arg \inf_{\hypothesis' \in \hypotheses \cap \Hypotheses^c(\{\example\})} \sigma^c(\hypothesis'; \hypotheses \cap \Hypotheses^c(\{\example\}), \hypothesis)$
denotes the set of candidate hypotheses most preferred by the learner. Note that $\Val_{\pref^c}(\hypotheses, \hypothesis, \hstar)$ can be infinite.

Now we can define the teaching dimension w.r.t. $\pref^c$ and the learner's initial hypothesis $\hinit$, in a similar way as  Eq.~\eqref{eq:sigmatd_fixedsigma}:
\begin{align}\label{eq:sigmatd_fixedsigma-infinite}
    \TD_{\Instances^c,\Hypotheses^c,\hinit}(\pref^c) = \sup_{\hstar} \Val_{\pref^c}(\Hypotheses^c, \hinit, \hstar).
\end{align}

Analogous to Eq~\eqref{eq:sigmatd}, for a family $\Sigma^c$ of 
preference functions $\sigma^c$ over the infinite domain, we define the teaching dimension w.r.t $\Sigma^c$ as the teaching dimension w.r.t. the \emph{best} $\pref^c$ in that family:
\begin{align}\label{eq:sigmatd-infinite}
        \Sigma^c_{}{\text -}\TD_{\Instances^c,\Hypotheses^c,\hinit} =  \inf_{\sigma^c \in \Sigma^c} \TD_{\Instances^c,\Hypotheses^c,\hinit}(\pref^c).
\end{align}

Similar to \defref{def:seq-col-free}, we consider collusion-free preference functions $\sigma^c$ for the infinite domain. Using this property, we use $\Sigma_{\CF}^c$ to denote the set of preference functions in the infinite domain that induce collusion-free teaching: 
\begin{align*}
    \Sigma_{\CF}^c = \{\sigma^c\mid \sigma^c \text{ is collusion-free}\}.
\end{align*}

\subsection*{Proof of the Theorem}

\begin{proof}[Proof of \thmref{theorem:equivelence-results-infinite}] We divide the proof in two parts.\\
\textbf{Part 1} $\PBTD(\Hypotheses^c) \leq \SigmaContGlobalTD_{\Instances^c,\Hypotheses^c,\hinit}$:\\
Consider $\sigmacstar \in \SigmaContGlobal$ such that $\TD_{\Instances^c,\Hypotheses^c,\hinit}(\sigmacstar)
= \SigmaContGlobalTD_{\Instances^c,\Hypotheses^c,\hinit}$. 
We build $\prec$ out of $\sigmacstar$ in the following way. For each pair of hypotheses $\hypothesis'' \neq \hypothesis'$, we define the $\prec$ as follows: 
\begin{itemize}
    \item If $\sigmacstar(\hypothesis'; \cdot, \cdot) < \sigmacstar(\hypothesis'', \cdot, \cdot)$, let $\hypothesis'' \prec \hypothesis'$. 
    \item If $\sigmacstar(\hypothesis'; \cdot, \cdot) > \sigmacstar(\hypothesis''; \cdot, \cdot)$, let $\hypothesis' \prec \hypothesis''$. 
\end{itemize}
For every $\hypothesis \in \Hypotheses^c$, we define the following:
\begin{align}
    \Hypotheses^c_{\sigmacstar, \hypothesis} = \{\hypothesis' \in \Hypotheses^c: \sigmacstar(\hypothesis'; \cdot, \cdot) > \sigmacstar(\hypothesis; \cdot, \cdot)\}.
\end{align}

Also, we rewrite the following definition below for ${\Hypotheses^c}_{\prec \hypothesis}$ which was introduced in the main paper: 
\begin{equation}
    {\Hypotheses^c}_{\prec \hypothesis} = \{\hypothesis' \in \Hypotheses^c: \hypothesis' \prec \hypothesis\}
\end{equation}


Based on the construction of $\prec$, it is clear that $\Hypotheses^c_{\sigmacstar, \hypothesis} = {\Hypotheses^c}_{\prec \hypothesis}$. This in turn allows us to establish that the teaching complexity for any hypothesis $\hypothesis \in \Hypotheses^c$ under $\sigma^c$ is same as teaching complexity under $\prec$. This means that the $\prec$ we have constructed from $\sigmacstar$ has $\PBTD(\Hypotheses^c, \prec) =  \TD_{\Instances^c,\Hypotheses^c,\hinit}(\sigmacstar)$. Also, since the preference function $\sigmacstar \in \SigmaContGlobal$  satisfies $\TD_{\Instances^c,\Hypotheses^c,\hinit}(\sigmacstar)
= \SigmaContGlobalTD_{\Instances^c,\Hypotheses^c,\hinit}$, we get the following result: $\PBTD(\Hypotheses^c, \prec) = \SigmaContGlobalTD_{\Instances^c,\Hypotheses^c,\hinit}$. From here, we can establish the proof as follows:
\begin{align*}
    \PBTD(\Hypotheses^c) = \inf_{\prec'} \PBTD(\Hypotheses^c, \prec') \leq \PBTD(\Hypotheses^c, \prec) = \SigmaContGlobalTD_{\Instances^c,\Hypotheses^c,\hinit}
\end{align*}

\noindent\textbf{Part 2} $\SigmaContGlobalTD_{\Instances^c,\Hypotheses^c,\hinit} \leq \PBTD(\Hypotheses^c)$:\\

In this part, we consider the relation $\prec^\star$ that achieves the lowest teaching complexity $\PBTD(\Hypotheses^c)$, i.e., $\PBTD(\Hypotheses^c, \prec^\star) = \PBTD(\Hypotheses^c)$.
For any strict partial relation on $\Hypotheses^c$, based on assumption of the theorem, there exists a $g$ such that for any two hypothesis $\hypothesis'' \neq \hypothesis'$, if $\hypothesis'' \prec^\star \hypothesis'$ we have $g(\hypothesis'') > g(\hypothesis')$. We define $\sigma^c(\hypothesis';\cdot, \cdot ) = g(\hypothesis')$.

Based on the construction of $\sigma^c$, it is clear that $\Hypotheses^c_{\sigma^c, \hypothesis} \supseteq {\Hypotheses^c}_{\prec^\star \hypothesis}$. This in turn allows us to establish that the teaching complexity for any hypothesis $\hypothesis \in \Hypotheses^c$ under $\sigma^c$ is at most the teaching complexity under $\prec^\star$. This means that the $\sigma^c$ we have constructed from $\prec^\star$ has $\TD_{\Instances^c,\Hypotheses^c,\hinit}(\sigma^c) \leq \PBTD(\Hypotheses^c, \prec^\star) = \PBTD(\Hypotheses^c)$. From here, we can establish the proof as follows:
\begin{align*}
    \SigmaContGlobalTD_{\Instances^c,\Hypotheses^c,\hinit} = \inf_{\sigma^{c'}} \TD_{\Instances^c,\Hypotheses^c,\hinit}(\sigma^{c'}) \leq \TD_{\Instances^c,\Hypotheses^c,\hinit}(\sigma^c) \leq \PBTD(\Hypotheses^c, \prec^\star) = \PBTD(\Hypotheses^c)
\end{align*}
\end{proof}

\section{Supplementary Materials for \secref{sec:seq-family-pref}: Proof of \thmref{thm:local-eq-GVS}
} \label{sec.appendix.seq-models_globalVS-local}

We divide the proof into three parts. The first part shows that the intersection of the two families is $\SigmaGlobal$. In part 2 and part 3 of the proof, we show that there exist  hypothesis classes, such that $\SigmaLocalTD_{\Instances,\Hypotheses,\hinit} > \SigmaGvsTD_{\Instances,\Hypotheses,\hinit}$, or $\SigmaLocalTD_{\Instances,\Hypotheses,\hinit} < \SigmaGvsTD_{\Instances,\Hypotheses,\hinit}$.



\subsection{Part 1}\label{sec.appendix.seq-models_globalVS-local.part1}
In this subsection, we provide the full proof for part 1 of \thmref{thm:local-eq-GVS}, i.e., $\SigmaGvs\cap \SigmaLocal=\SigmaGlobal$. 
Intuitively, observe that the input domains between $\sigma_\lcl \in \SigmaLocal$ and $\sigma_\gvs \in \SigmaGvs$ overlap at the domain of the first argument, which is the one taken by $\sigma_\glbl$. Therefore, $\forall \sigma \in \SigmaGlobal, \sigma \in \SigmaGvs \cap \SigmaLocal$. We formalize this idea in the proof below. 

\begin{proof}[Proof of Part 1 of \thmref{thm:local-eq-GVS}]
Assume $\sigma \in \SigmaLocal \cap \SigmaGvs$. Then, by the definitions of $\SigmaLocal$ and $\SigmaGvs$, we get 
\begin{enumerate}[(i)]\denselist
    \item $\exists g^a$, s.t. $\forall \hypothesis, \hypothesis' \in \Hypotheses: \sigma( \hypothesis' ; \cdot, \hypothesis)  = g^a(\hypothesis', \hypothesis)$, and 
    \item $\exists g^b$, s.t. $\forall \hypothesis' \in \Hypotheses, \hypotheses \subseteq \Hypotheses: \sigma( \hypothesis' ; \hypotheses, \cdot)  = g^b(\hypothesis', \hypotheses)$ 
\end{enumerate}
Now consider $\hypothesis', \hypothesis^1, \hypothesis^2 \in \Hypotheses$, and $\hypotheses^1, \hypotheses^2 \subseteq \Hypotheses$. According to (i), $\sigma(\hypothesis'; \hypotheses^1, \hypothesis^1) = \sigma(\hypothesis'; \hypotheses^2, \hypothesis^1)$. Also, according to (ii), $\sigma(\hypothesis'; \hypotheses^2, \hypothesis^1) = \sigma(\hypothesis'; \hypotheses^2, \hypothesis^2)$. This indicates that $\forall \hypothesis', \hypothesis^1, \hypothesis^2 \in \Hypotheses$ and $\forall \hypotheses^1, \hypotheses^2 \subseteq \Hypotheses$, we have $\sigma(\hypothesis'; \hypotheses^1, \hypothesis^1) = \sigma(\hypothesis'; \hypotheses^2, \hypothesis^2)$. In other words, there exists $g^c: \Hypotheses \to \reals$, s.t. $\forall \hypothesis' \in \Hypotheses: \sigma(\hypothesis';\cdot, \cdot) = g^c(\hypothesis')$. Thus, $\sigma \in \SigmaGlobal$.
\end{proof}
   

\subsection{Part 2}\label{sec.appendix.seq-models_globalVS-local.part2}
Next, we show that there exists $(\Hypotheses, \Instances)$, such that $\forall \hinit\in \Hypotheses$, $\SigmaLocalTD_{\Instances,\Hypotheses,\hinit} > \SigmaGvsTD_{\Instances,\Hypotheses,\hinit}$. To prove this statement, we first establish the following lemma.

\begin{lemma}\label{lem:local-rtd-1}
For any $\Hypotheses$, $\Instances$, and $\hinit\in \Hypotheses$, if $\SigmaLocalTD_{\Instances,\Hypotheses,\hinit} = 1$, then $\SigmaGlobalTD_{\Instances,\Hypotheses,\hinit} = 1$.
\end{lemma}
\begin{proof}[Proof of \lemref{lem:local-rtd-1}]
If $\SigmaLocalTD_{\Instances,\Hypotheses,\hinit} = 1$, there should be some $\sigma_{\lcl} \in \SigmaLocal$ such that $\TD_{\Instances,\Hypotheses,\hinit}(\sigma_\lcl) = 1$. Now consider $\sigma_{\glbl}$ such that $\forall \hypothesis', \sigma_\glbl(\hypothesis';\cdot,\cdot) = \sigma_{\lcl}(\hypothesis';\cdot,\hinit)$. If $\Teacher_{\sigma_{\lcl}}$ is the best teacher for $\sigma_{\lcl}$, then $\forall \hypothesis \in \Hypotheses: |T_{\sigma_{\lcl}}(\hypothesis)| = 1$.  For a given hypothesis $\hypothesis$, let us denote the single teaching example in the sequence $T_{\sigma_{\lcl}}(\hypothesis)$ as $z^\star_{h}$.
Then, this indicates that $\argmin_{\hypothesis' \in \Hypotheses(\{z^\star_{h}\})} \sigma_{\lcl}(\hypothesis'; \cdot,\hinit) = \{\hypothesis\}$. 

Subsequently, $\argmin_{\hypothesis' \in \Hypotheses(\{z^\star_{h}\})} \sigma_{\glbl}(\hypothesis';\cdot,\cdot) = \{\hypothesis\}$. In other words, $\Teacher_{\sigma_{\lcl}}$ is also a teacher for $\sigma_{\glbl}$. This indicates that $\SigmaGlobalTD_{\Instances,\Hypotheses,\hinit} = \TD_{\Instances,\Hypotheses,\hinit}(\sigma_\glbl) = 1$.
%
%
\end{proof}

Now we are ready to provide the proof for part 2.
\begin{proof}[Proof of Part 2 of \thmref{thm:local-eq-GVS}]
We identify $\Hypotheses$, $\Instances$, $\hinit$, where $\SigmaGvsTD_{\Instances,\Hypotheses,\hinit} = 1$ and $\SigmaGlobalTD_{\Instances,\Hypotheses,\hinit} = 2$. 
\tableref{tab:batch_model_example_h2} illustrates such a class.
Here, since $\SigmaGlobalTD_{\Instances,\Hypotheses,\hinit}  = 2$, then by \lemref{lem:local-rtd-1}, it must hold that $\SigmaLocalTD_{\Instances,\Hypotheses,\hinit} > 1 = \SigmaGvsTD_{\Instances,\Hypotheses,\hinit}$.
\end{proof}

\subsection{Part 3}\label{sec.appendix.seq-models_globalVS-local.part3}


In this section, we construct hypothesis classes where $\SigmaLocalTD_{\Instances,\Hypotheses,\hinit} < \SigmaGvsTD_{\Instances,\Hypotheses,\hinit}$, and show that this gap could be made arbitrarily large. We first show that for the powerset hypothesis class of size 7, $\SigmaGvsTD_{\Instances,\Hypotheses,\hinit} \geq 4$ (\lemref{lem:nctd-powerset-kby2}) and $\SigmaLocalTD_{\Instances,\Hypotheses,\hinit} \leq 3$ (\lemref{lem:lvs-H7-leq3}). Based on this result, we then provide a constructive procedure which extends the gap to be $\SigmaGvsTD_{\Instances,\Hypotheses,\hinit} - \SigmaLocalTD_{\Instances,\Hypotheses,\hinit} \geq 2^{m-1}$ for any choince of $m$ 
(\lemref{lem:k-2k}).


 

\begin{lemma}[Based on Theorem 23 of \cite{pmlr-v98-kirkpatrick19a}]\label{lem:nctd-powerset-kby2}
Consider the powerset hypothesis class of size $k$, i.e., $\Hypotheses=\{0,1\}^k$. Then, $\NCTD \geq \ceil{k/2}$.
\end{lemma}
\begin{proof}[Proof of \lemref{lem:nctd-powerset-kby2}]
First we make the following observation: If $T$ is a non-clashing teacher and $\hypothesis, \hypothesis' \in \Hypotheses$ where $\hypothesis = \hypothesis' \triangle \instance$ (i.e., these two hypotheses only differ in their label on one instance), it must be the case that $(\instance, \hypothesis(\instance)) \in T(\hypothesis)$, or $(\instance, \hypothesis'(\instance)) \in T(\hypothesis')$. This holds by nothing that since $\hypothesis$, and $\hypothesis'$ are only different on $\instance$, if $\instance$ is absent in their teaching sequences, this would lead to violation of the non-clashing property of the teacher.


Next we apply this observation on the powerset $k$ hypothesis class where $\Hypotheses$ consists of all hypotheses which have length $k$. This indicates that for every $\hypothesis \in \Hypotheses$ and $0 \leq j \leq (k-1)$, all the $k$ variants satisfy $\hypothesis \triangle \instance_j \in \Hypotheses$. By using this property for all pairs $\hypothesis$ and $\hypothesis \triangle \instance_j$ with $0 \leq j \leq (k-1)$, we can drive $\sum_{i = 0}^{2^k-1} |T(\hypothesis_i)| \geq \frac{k \cdot 2^k}{2}$. By applying the pigeonhole principle, this indicates that there exist an $\hypothesis \in \Hypotheses$ where $|T(\hypothesis)| \geq \frac{k}{2}$. Hence,  $\NCTD(\Hypotheses) \geq \ceil{\frac{k}{2}}$.
\end{proof}

For the powerset hypothesis class of size $7$,  based on \lemref{lem:nctd-powerset-kby2} we get $\SigmaGvsTD_{\Instances,\Hypotheses,\hinit} = 
\NCTD(\Hypotheses) \geq 4$. 
%
Next, in \lemref{lem:lvs-H7-leq3}, we show that for the powerset hypothesis class of size $7$, we can find a preference function $\sigma \in \SigmaLocal$ such that $\SigmaLocalTD_{\Instances,\Hypotheses,\hinit} \leq 3$.


\begin{lemma}\label{lem:lvs-H7-leq3}
Consider the powerset hypothesis class of size $7$, i.e., $\Hypotheses=\{0,1\}^7$. Then, $\SigmaLocalTD_{\Instances,\Hypotheses,\hinit} \leq 3$.
\end{lemma}
\begin{proof}[Proof of \lemref{lem:lvs-H7-leq3}]
\looseness-1We construct a $\sigma \in \SigmaLocal$ such that $\TD_{\Instances,\Hypotheses,\hinit}(\pref) = 3$; see Figure~\ref{fig:Local_is_less_than_NCTD_7} and Table~\ref{tab:Local_is_less_than_NCTD_7}. Intuitively, we construct a tree of hypotheses with branching factor $7$ at the top level, branching factor of $6$ at the next level, and so on. Here, each branch corresponds to one teaching example, and each path from $\hinit$ to $\hypothesis\in\Hypotheses$ corresponds to a teaching sequence $\Teacher_\lcl(h)$. We need a tree of depth at most $3$ to include all the $2^7=128$ hypotheses to be taught as nodes in the tree. This gives us a construction of $\sigma \in \SigmaLocal$ function such that $\TD_{\Instances,\Hypotheses,\hinit}(\pref) = 3$, which implies that $\SigmaLocalTD_{\Instances,\Hypotheses,\hinit}(\pref) \leq 3$ thereby completing the proof.
\end{proof}

\paragraph{Remark 1.} Figure~\ref{fig:Local_is_less_than_NCTD_7} and Table~\ref{tab:Local_is_less_than_NCTD_7} only specify the preference relations induced by $\sigma$ function and do not specify the exact values of the preference function. One can easily construct the actual $\sigma$ function from these relations as follows: (i) for each hypothesis $\hypothesis \in \Hypotheses$, set $\sigma(\hypothesis; \cdot, \hypothesis) = 0$, and  (ii) for all pair of hypotheses $\hypothesis \neq \hypothesis' \in \Hypotheses$, set $\sigma(\hypothesis'; \cdot; \hypothesis) \in (0, |\Hypotheses| - 1]$ based on the rank of $\hypothesis'$ in the preference list of $\hypothesis$.

\paragraph{Remark 2.} The preference function $\sigma \in \SigmaLocal$ we constructed in Lemma~\ref{lem:lvs-H7-leq3} also belong to the family $\SigmaWinStayLoseShift$, i.e.,   $\sigma \in \SigmaLocal \cap \SigmaWinStayLoseShift$. 
\vspace{6mm}

In the following lemma, we provide a result for $\SigmaLocal$ family of preference functions that allow us to reason about the teaching complexity of larger powerset classes based on the teaching complexity of a smaller powerset class.

\begin{lemma}\label{lem:k-2k}
Let $\Hypotheses^k = \{0,1\}^k$ and $\Instances^k = \{\instance^k_0, \ldots, \instance^k_{k-1}\}$ denote the hypothesis class and set of instances for the powerset of size $k$. Similarly, let $\Hypotheses^{2k} = \{0,1\}^{2k}$ and $\Instances^{2k} = \{\instance^{2k}_0, \ldots, \instance^{2k}_{2k-1}\}$ denote the hypothesis class and set of instances for powerset of size $2k$. W.l.o.g., let $\hypothesis^k_0 \in \Hypotheses^k$ and $\hypothesis^{2k}_0 \in \Hypotheses^{2k}$ be the initial hypotheses with label zero over all the instances. For the powerset of size $k$, consider any  $\sigma^{k} \in \SigmaLocal^{k} \cap \SigmaWinStayLoseShift^{k}$ with the following properties: (i) $\forall \hypothesis^k \in \Hypotheses^k$, $\sigma^k(\hypothesis^k; \cdot, \hypothesis^k) = 0$, and  (ii) $\forall \hypothesis^k \neq \hypothesis^{k'} \in \Hypotheses^k$, $\sigma^k(\hypothesis^{k'}; \cdot, \hypothesis^k) \in (0, |\Hypotheses^k| - 1]$. Then, for the powerset of size $2k$, there exists  $\sigma^{2k}  \in \SigmaLocal^{2k} \cap \SigmaWinStayLoseShift^{2k}$ with the above two properties and the following teaching complexity: $\TD_{\Instances^{2k},\Hypotheses^{2k},\hypothesis^{2k}_0}(\pref^{2k}) \leq 2 \cdot \TD_{\Instances^{k},\Hypotheses^{k},\hypothesis^k_0}(\pref^{k})$.
\end{lemma}

\begin{proof}[Proof of \lemref{lem:k-2k}]
We begin the proof by introducing a few definitions. First, we denote a special set of ``pivot'' hypotheses $P^{2k} \subset \Hypotheses^{2k}$ obtained by concatenating $k$ zeros after the powerset of size $k$, i.e., $\{0,1\}^k \times \{0\}^k$. Second, for $\hypothesis^{2k} \in \Hypotheses^{2k}$ and $0 \leq i <  j \leq 2k-1$, we define $\hypothesis^{2k}_{i,\dots, j} = (\hypothesis^{2k}(\instance_i), \ldots, \hypothesis^{2k}(\instance_j))$ to be the projection of $\hypothesis^{2k}$ to the subspace of instances indexed by $i,\ldots,j$. Third, for every pivot hypothesis $p^{2k} \in P^{2k}$, we define a group of hypotheses in $\Hypotheses^{2k}$ belonging to $p^{2k}$ by the set $P^{2k}_{p} = \{\hypothesis^{2k} \in \Hypotheses^{2k}: \hypothesis^{2k}_{0, ... ,k-1} = p^{2k}_{0, ... ,k-1}\}$. 

Next, we will construct a preference function $\sigma^{2k}  \in \SigmaLocal^{2k} $  which will satisfy all the requirements. We begin by  initializing $\pref_{2k}$ for all $\hypothesis^{2k}, \hypothesis^{2k'} \in \Hypotheses^{2k}$  as follows:
	\begin{align*}
		\pref^{2k}(\hypothesis^{2k'}; \cdot, \hypothesis^{2k})  = 
  		\begin{cases}
    		0 , & \text{if } \hypothesis^{2k} = \hypothesis^{2k'} \\
    		|\Hypotheses^{2k}| - 1, & \text{otherwise } \\
  		\end{cases}
	\end{align*}

Then, we assign preferences for hypotheses among the set $P^{2k}$.  For every $p^{2k}, p^{2k'} \in P^{2k}$ we assign $\pref^{2k}(p^{2k'}; \cdot, p^{2k}) = \pref^{k}(p^{2k'}_{0, ... ,k-1}; \cdot, p^{2k}_{0, ... ,k-1})$.

Finally, we assign preferences for a pivot hypothesis $p^{2k} \in P^{2k}$ to hypotheses in the set $P^{2k}_{p}$. For all $p^{2k} \in P^{2k}$, consider all pairs $\hypothesis^{2k} \neq \hypothesis^{2k'} \in P^{2k}_{p}$ and assign $\pref^{2k}(\hypothesis^{2k'}; \cdot, \hypothesis^{2k}) = \pref^k(\hypothesis^{2k'}_{k, ... ,2k-1}; \cdot, \hypothesis^{2k}_{k, ... ,2k-1}) + |\Hypotheses^k|$.

Based on this construction, it is easy to verify that $\pref^{2k} \in \SigmaLocal^{2k} \cap \SigmaWinStayLoseShift^{2k}$ and satisfy the two properties stated in the lemma. Next, we upper bound the teaching complexity $\TD_{\Instances^{2k},\Hypotheses^{2k},\hypothesis^{2k}_0}(\pref^{2k})$ based on the following two observations:
\begin{itemize}
	\item Starting from $\hypothesis_0^{2k}$, to teach any hypothesis $p^{2k} \in P^{2k}$, we need at most $\TD_{\Instances^{k},\Hypotheses^{k},\hypothesis^k_0}(\pref^{k})$ number of examples.
	\item Starting from a pivot $p^{2k} \in P^{2k}$, to teach any hypothesis $\hypothesis^{2k} \in P^{2k}_{p}$, we need at most $\TD_{\Instances^{k},\Hypotheses^{k},\hypothesis^k_0}(\pref^{k})$ number of examples. 
\end{itemize}
Based on the above two observations, we can show that the teaching complexity for $\sigma^{2k}$ satisfies $\TD_{\Instances^{2k},\Hypotheses^{2k},\hypothesis^{2k}_0}(\pref^{2k}) \leq 2 \cdot \TD_{\Instances^{k},\Hypotheses^{k},\hypothesis^k_0}(\pref^{k})$.
\end{proof}
Finally, now we provide the proof for part 3 of \thmref{thm:local-eq-GVS}.
\begin{proof}[Proof of Part 3 of \thmref{thm:local-eq-GVS}]
Fix a positive integer $m$. Let $(\Hypotheses, \Instances)$ represent the powerset class of size $7\cdot2^m$. Starting with a powerset of size $7$ where we have the teaching complexity result from \lemref{lem:lvs-H7-leq3}, and then iteratively applying \lemref{lem:k-2k} $m$ times, we can easily conclude that $\SigmaLocalTD_{\Hypotheses, \Instances, \hinit} \leq 3 \cdot 2^m$. Moreover, based on \lemref{lem:nctd-powerset-kby2}, we know that $\SigmaGvsTD_{\Instances, \Hypotheses, \hinit} = \NCTD(\Hypotheses) \geq 7\cdot2^{m - 1}$. Thus, we have $\SigmaGvsTD_{\Instances, \Hypotheses, \hinit} - \SigmaLocalTD_{\Instances, \Hypotheses, \hinit} \geq 2^{m - 1}$ and the proof is complete.
\end{proof}
\begin{figure}[!h]
    \centering
    \begin{subfigure}[b]{\textwidth}
        \centering
        \includegraphics[width = 0.75\linewidth]{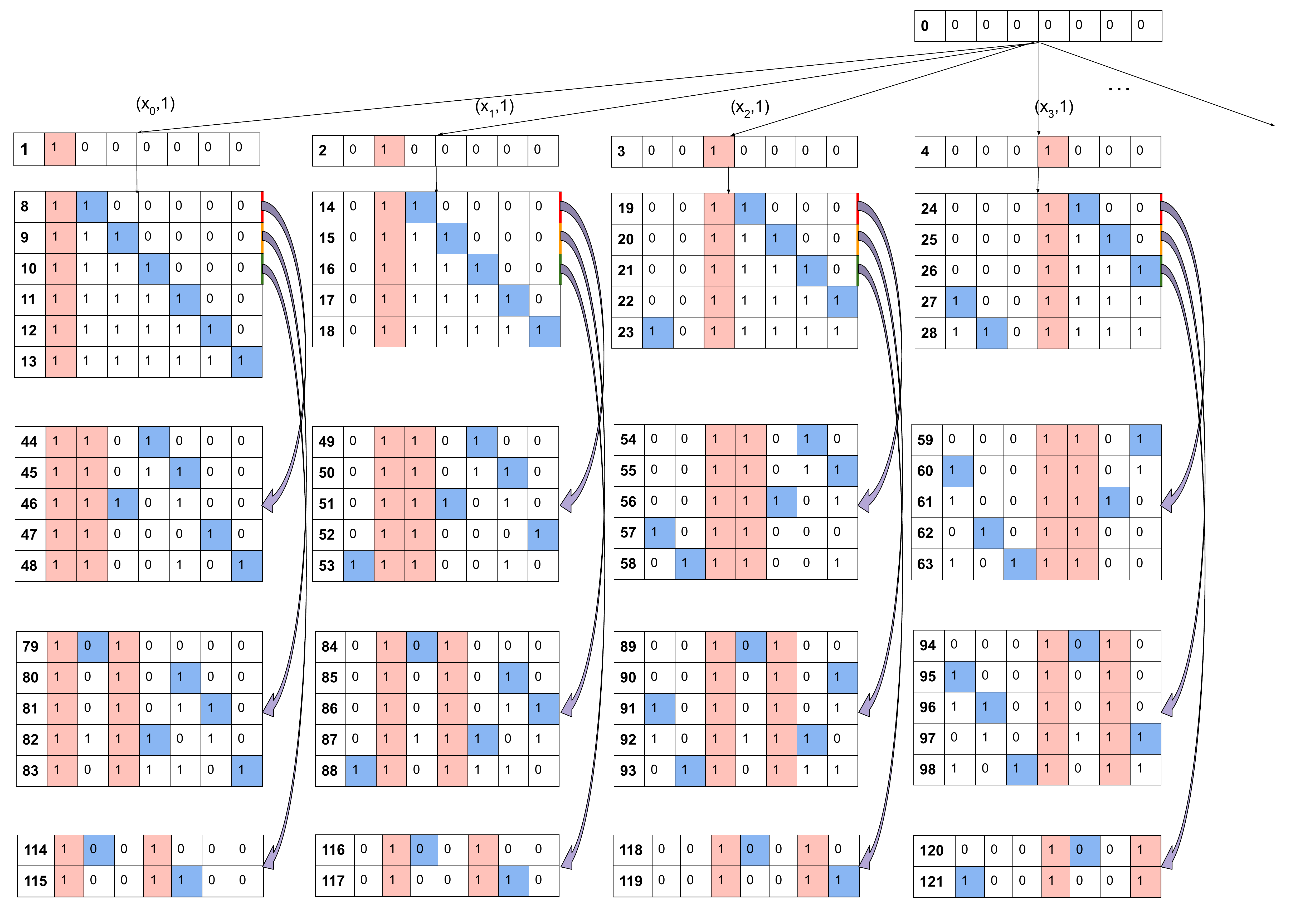}
        \caption{This figure is representing teaching sequences for the first four direct children of $\hinit$ (top four most preferred hypothesis of $\hinit$ after $\hinit$) and all of their children.}
    \end{subfigure}
    \centering
    \begin{subfigure}[b]{\textwidth}
        \centering
        \includegraphics[width = 0.75\linewidth]{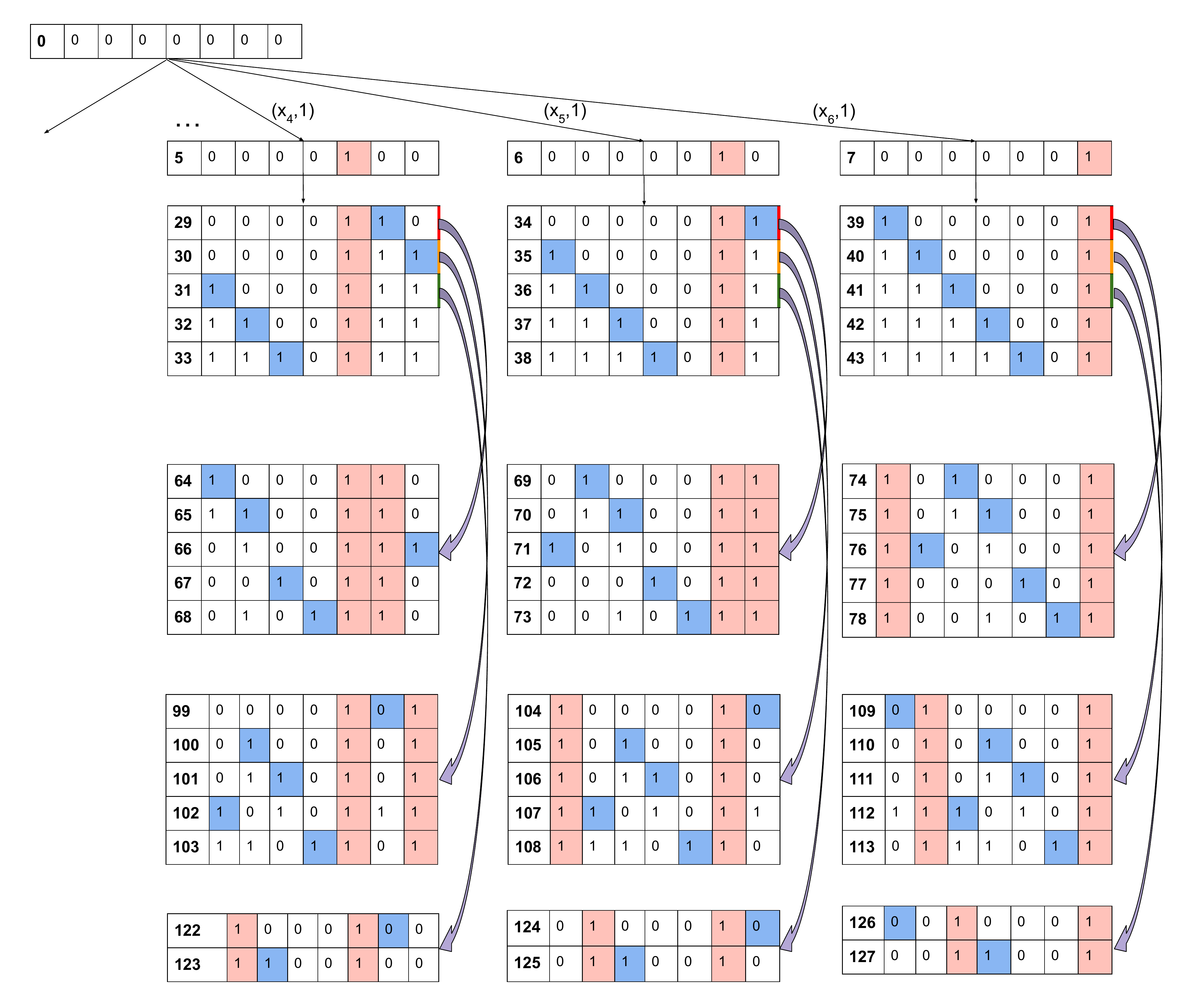}
        \caption{This figure is representing teaching sequences for the next three direct children of $\hinit$ (the next three most preferred hypothesis of $\hinit$) and all of their children.}
    \end{subfigure}
    \caption{Details of teaching sequences for a preference function $\sigma\in \SigmaLocal$, where $\TD_{\Instances,\Hypotheses,\hinit}(\pref) = 3$ for powerset $k=7$ class. For any hypothesis, the blue cell represents the last teaching example in the teaching sequence, and the red cells represent the rest of the teaching sequence. See Table~\ref{tab:Local_is_less_than_NCTD_7} for further details.
    }
    \label{fig:Local_is_less_than_NCTD_7}
\end{figure}
\clearpage

\begin{table}[!h]
\centering
\scalebox{0.95}{
\begin{tabular}{ c | c  ||  c | c }

$\hypothesis $ & $\Instances$ & \textbf{Preferences induced by $\sigma(\cdot ; \cdot, \hypothesis)$} & \textbf{Teaching sequence}\\
\hline
\underline{$\hypothesis_0$} & 0 0 0 0 0 0 0 &
$\hypothesis_0 > \hypothesis_1 > \hypothesis_2 > \hypothesis_3 > \hypothesis_4 >$\newline$ \hypothesis_5  > \hypothesis_6 > \hypothesis_7 >$ others &
$\paren{(\instance_0, 0)}$
\\
\hline
$\hypothesis_1$ & 1 0 0 0 0 0 0 &
$\hypothesis_1 > \hypothesis_8 > \hypothesis_9 > \hypothesis_{10} > \hypothesis_{11} > \hypothesis_{12} > \hypothesis_{13} $ others &
$\paren{(\instance_0, 1)}$
\\
\hline
$\hypothesis_8$ & 1 1 0 0 0 0 0 &
$\hypothesis_{8} > \hypothesis_{44} > \hypothesis_{45} > \hypothesis_{46} > \hypothesis_{47} > \hypothesis_{48} >$ others &
$\paren{(\instance_0, 1), (\instance_1, 1)}$
\\
$\hypothesis_{9}$ & 1 1 1 0 0 0 0 &
$\hypothesis_{9} > \hypothesis_{79} > \hypothesis_{80} > \hypothesis_{81} > \hypothesis_{82} > \hypothesis_{83} >$ others &
$\paren{(\instance_0, 1), (\instance_2, 1)}$
\\
$\hypothesis_{10}$ & 1 1 1 1 0 0 0 &
$\hypothesis_{10} > \hypothesis_{114} > \hypothesis_{115} >$ others &
$\paren{(\instance_0, 1), (\instance_3, 1)}$
\\
$\hypothesis_{11}$ & 1 1 1 1 1 0 0 &
$\hypothesis_{11} >$ others &
$\paren{(\instance_0, 1), (\instance_4, 1)}$
\\
$\hypothesis_{12}$ & 1 1 1 1 1 1 0 &
$\hypothesis_{12} >$ others &
$\paren{(\instance_0, 1), (\instance_5, 1)}$
\\
$\hypothesis_{13}$ & 1 1 1 1 1 1 1 &
$\hypothesis_{13} >$ others &
$\paren{(\instance_0, 1), (\instance_6, 1)}$
\\
\hline
$\hypothesis_{44}$ & 1 1 0 1 0 0 0 &
$\hypothesis_{44} >$ others &
$\paren{(\instance_0, 1), (\instance_1, 1), (\instance_3, 1)}$
\\
$\hypothesis_{45}$ & 1 1 0 1 1 0 0 &
$\hypothesis_{45}>$ others &
$\paren{(\instance_0, 1), (\instance_1, 1), (\instance_4, 1)}$
\\
$\hypothesis_{46}$ & 1 1 1 0 1 0 0 &
$\hypothesis_{46} >$ others &
$\paren{(\instance_0, 1), (\instance_1, 1), (\instance_2, 1)}$
\\
$\hypothesis_{47}$ & 1 1 0 0 0 1 0 &
$\hypothesis_{47} >$ others &
$\paren{(\instance_0, 1), (\instance_1, 1), (\instance_5, 1)}$
\\
$\hypothesis_{48}$ & 1 1 0 0 1 0 1 &
$\hypothesis_{48} >$ others &
$\paren{(\instance_0, 1), (\instance_1, 1), (\instance_6, 1)}$
\\
\hline
$\hypothesis_{79}$ & 1 0 1 0 0 0 0 &
$\hypothesis_{79} >$ others &
$\paren{(\instance_0, 1), (\instance_2, 1), (\instance_1, 0)}$
\\
$\hypothesis_{80}$ & 1 0 1 0 1 0 0 &
$\hypothesis_{80}>$ others &
$\paren{(\instance_0, 1), (\instance_2, 1), (\instance_4, 1)}$
\\
$\hypothesis_{81}$ & 1 0 1 0 1 1 0 &
$\hypothesis_{81} >$ others &
$\paren{(\instance_0, 1), (\instance_2, 1), (\instance_5, 1)}$
\\
$\hypothesis_{82}$ & 1 1 1 1 0 1 0 &
$\hypothesis_{47} >$ others &
$\paren{(\instance_0, 1), (\instance_2, 1), (\instance_3, 1)}$
\\
$\hypothesis_{83}$ & 1 0 1 1 1 0 1 &
$\hypothesis_{83} >$ others &
$\paren{(\instance_0, 1), (\instance_2, 1), (\instance_6, 1)}$
\\
\hline
$\hypothesis_{114}$ & 1 0 0 1 0 0 0 &
$\hypothesis_{114} >$ others &
$\paren{(\instance_0, 1), (\instance_3, 1), (\instance_1, 0)}$
\\
$\hypothesis_{115}$ & 1 0 0 1 1 0 0 &
$\hypothesis_{115}>$ others &
$\paren{(\instance_0, 1), (\instance_3, 1), (\instance_4, 1)}$
\\

\end{tabular}}
\vspace*{2mm}
\caption{More details about Figure~\ref{fig:Local_is_less_than_NCTD_7} -- This table lists down all the hypotheses in the left branch of the tree. For each of these hypotheses, it shows the preferences induced by the $\sigma$ function, as well as the teaching sequence to teach the hypothesis. Consider $\hypothesis_9$: We have the preferences induced by $\sigma(\cdot; \cdot, \hypothesis_9)$ as $\hypothesis_{9} > \hypothesis_{79} > \hypothesis_{80} > \hypothesis_{81} > \hypothesis_{82} > \hypothesis_{83} > \textnormal{others}$, and the teaching sequence for $\hypothesis_9$ is $((\instance_0, 1), (\instance_2, 1))$.}
\label{tab:Local_is_less_than_NCTD_7}
\end{table}
\section{Supplementary Materials for \secref{sec:newresults:sub-add}: Proof of \lemref{lem:additional:subadditive}}\label{appendix:newresults:sub-add}

\begin{proof}[Proof of \lemref{lem:additional:subadditive}] The proof is divided into two parts: proving sub-additivity and proving strict sub-additivity.

\paragraph{Part 1 of the proof on sub-additivity.} 
Let $\Sigma^a$ and  $\Sigma^b$ denote the $\SigmaWinStayLoseShift$ families of preference functions over the hypothesis classes $\Hypotheses^a$ and $\Hypotheses^b$ respectively. We use the notation $\Sigma^{ab}$ to denote the $\SigmaWinStayLoseShift$ family of preference functions over the disjoint union of the hypothesis classes given by $\Hypotheses^a \uplus \Hypotheses^b$. We will denote a preference function in these families as $\pref^a \in \Sigma^a$, $\pref^b \in \Sigma^b$, and $\pref^{ab} \in \Sigma^{ab}$. Furthermore, let us denote the best preference functions that achieve the minimal teaching complexity as follows:
\begin{align*}
 \pref^{a\star} &\in \argmin_{\pref^a \in \Sigma^a} \TD_{\Instances^a, \Hypotheses^a, \hypothesis^a_0}(\pref^a)\\   
 \pref^{b\star} &\in \argmin_{\pref^b \in \Sigma^b} \TD_{\Instances^b, \Hypotheses^b, \hypothesis^b_0}(\pref^b)
\end{align*}

We will establish the proof by constructing a preference function $\pref^{ab} \in \Sigma^{ab}$ such that $\TD_{\Instances^{ab}, \Hypotheses^{ab}, \hypothesis^{ab}_0}(\pref^{ab}) \leq \TD_{\Instances^a, \Hypotheses^a, \hypothesis^a_0}(\pref^{a\star}) + \TD_{\Instances^b, \Hypotheses^b, \hypothesis^b_0}(\pref^{b\star})$. For any hypotheses $\hypothesis^a, \hypothesis^{a'} \in \Hypotheses^a$, hypotheses $\hypothesis^b, \hypothesis^{b'} \in \Hypotheses^b$, and version spaces $\hypotheses^a \subseteq \Hypotheses^a$, $\hypotheses^b \subseteq \Hypotheses^b$, we construct a preference function $\pref^{ab}$ as follows:
\begin{align*}
    \pref^{ab}(\hypothesis^{a'} \uplus \hypothesis^{b'}; \hypotheses^{ab}, \hypothesis^a \uplus \hypothesis^b) = \pref^{a\star}(\hypothesis^{a'}; \hypotheses^{a}, \hypothesis^a) + \pref^{b\star}(\hypothesis^{b'}; \hypotheses^{b}, \hypothesis^b).
\end{align*}
where the version space $\hypotheses^{ab}$ for the disjoint union that is reachable during the teaching phase always takes the form as $\hypotheses^{ab} = \hypotheses^{a} \uplus \hypotheses^{b}$.


Let the starting hypothesis in $\Hypotheses^a \uplus \Hypotheses^b$ be $\hypothesis^a_0 \uplus \hypothesis^b_0$, and the target hypothesis be $\hypothesis^{a\star} \uplus \hypothesis^{b\star}$. If we provide a sequence of labeled instances only from $\Instances^{b}$, the version space left would always be $\Hypotheses^{a} \uplus \hypotheses^{b}$, where $\hypotheses^{b}$ is the version space of $\Hypotheses^{b}$ after providing these examples from $\Instances^{b}$. 
%
Now, we know that $\argmin_{\hypothesis^{a'}} \pref^{a\star}(\hypothesis^{a'} ; \Hypotheses^a, \hypothesis^a_0) = \{\hypothesis^a_0\}$;  consequently, $\hypothesis^{a'}$ must stay on $\hypothesis^a_0$. From here, we can conclude that if the teacher provides teaching sequence of $\hypothesis^{b\star}$ to the learner, the learner can be steered to $\hypothesis^a_0 \uplus \hypothesis^{b\star}$ with at most $\TD_{\Instances^b, \Hypotheses^b, \hypothesis^b_0}(\pref^{b\star})$ examples.

Let us denote the version space when the learner reaches $\hypothesis^a_0 \uplus \hypothesis^{b\star}$ as $\Hypotheses^a \uplus \hypotheses^b$. Next, if we continue providing labeled instances only from $\Instances^a$, the version space left would always be $\hypotheses^a \uplus \hypotheses^b$, where $\hypotheses^a$ is the version space of $\Hypotheses^a$ after providing these examples from $\Instances^a$. 
Moreover, as $\argmin_{\hypothesis^{b'}}  \pref^{b\star}(\hypothesis^{b'} ;  \hypotheses^b, \hypothesis^{b\star}) = \{\hypothesis^{b\star}\}$; consequently, $\hypothesis^{b'}$ must stay on $\hypothesis^{b\star}$. Therefore, if the teacher provides teaching sequence of $\hypothesis^{a\star}$ to the learner, the learner can be steered from $\hypothesis^a_0 \uplus \hypothesis^{b\star}$ to $\hypothesis^{a\star} \uplus \hypothesis^{b\star}$ with at most $\TD_{\Instances^a, \Hypotheses^a, \hypothesis^a_0}(\pref^{a\star})$ examples. Thus, we can conclude the following:
\begin{align*}
\Sigmatd{\Instances^a \cup \Instances^b, \Hypotheses^a \uplus \Hypotheses^b, \hypothesis_0^a \uplus \hypothesis_0^b} \leq \TD_{\Instances^a \cup \Instances^b, \Hypotheses^a \uplus \Hypotheses^b, \hypothesis^a_0  \uplus \hypothesis^b_0}(\pref^{ab}) \leq \Sigmatd{\Instances^a, \Hypotheses^a,\hypothesis_0^a} + \Sigmatd{\Instances^b, \Hypotheses^b,\hypothesis_0^b}.
\end{align*}

\paragraph{Part 2 of the proof on strict sub-additivity.} 
Next we prove that the family of preference functions $\Sigma := \SigmaWinStayLoseShift$ is strictly  sub-additive, i.e., there exist hypothesis classes where the relation $\leq$ holds with $<$.
Let $(\Hypotheses^3$, $\Instances^3)$ represent the powerset class of size $3$, and $(\Hypotheses^4$, $\Instances^4)$ represent the powerset class of size $4$. Also, w.l.o.g. let $\hypothesis^3_0$ to be the hypothesis with label zero over all the instances $\Instances^3$, and $\hypothesis^4_0$ to be the hypothesis with label zero over all the instances $\Instances^4$. In the following, we will prove that $\Sigmatd{\Instances^3 \cup \Instances^4, \Hypotheses^3 \uplus \Hypotheses^4, \hypothesis^3_0 \uplus \hypothesis^4_0} < \Sigmatd{\Instances^3, 
\Hypotheses^3, \hypothesis^3_0} + \Sigmatd{\Instances^4, \Hypotheses^4, \hypothesis^4_0}$.

Let $\Sigma^3$ denote the $\SigmaWinStayLoseShift$ family for the powerset class of size $3$ and $\sigma^3 \in \Sigma^3$ denote a preference function in this family. We will first show that $\forall \sigma^3 \in \Sigma^3: \TD_{\Instances^3, \Hypotheses^3, \hypothesis^3_0}(\sigma^3) \geq 2$. For the sake of contradiction, assume there exists $\sigma^3$ such that $\TD_{\Instances^3, \Hypotheses^3, \hypothesis^3_0}(\sigma^3) = 1$. Since $\argmin \sigma^3(\cdot; \cdot, \hypothesis^3_0) = \{ \hypothesis^3_0 \}$, we know that the teacher needs to provide an example which is not consistent with $\hypothesis^3_0$ to make the learner output a hypothesis different from $\hypothesis^3_0$. Since $|\Instances^3| = 3$, with teaching sequence of size $1$, only a maximum of three hypotheses $\hypothesis \neq \hypothesis^3_0$ can be taught. However, $|\Hypotheses^3| = 8$ which indicates that $\Sigmatd{\Instances^3, \Hypotheses^3, \hypothesis^3_0} \geq 2$. Similarly, we can establish that  $\Sigmatd{\Instances^4, \Hypotheses^4, \hypothesis^4_0} \geq 2$.

We will complete the proof by showing that $\Sigmatd{\Instances^3 \cup \Instances^4, \Hypotheses^3 \uplus \Hypotheses^4, \hypothesis^3_0 \uplus \hypothesis^4_0} \leq 3$.  We note that the hypothesis class $\Hypotheses^3 \uplus \Hypotheses^4$ over $\Instances^3 \cup \Instances^4$ is equivalent to the powerset class of size $7$ which we denote as $(\Hypotheses^7$, $\Instances^7)$.  In Lemma~\ref{lem:lvs-H7-leq3}, we have proved that $\Sigmatd{\Instances^7, \Hypotheses^7, \hypothesis_0^7} \leq 3$; a careful inspection of the constructed $\sigma$ in Lemma~\ref{lem:lvs-H7-leq3} reveals that it belongs to  $\SigmaWinStayLoseShift$ family which in turn implies $\Sigmatd{\Instances^7, \Hypotheses^7, \hypothesis_0^7} \leq 3$. Hence, we have $\Sigmatd{\Instances^3 \cup \Instances^4, \Hypotheses^3 \uplus \Hypotheses^4, \hypothesis^3_0 \uplus \hypothesis^4_0} \leq 3$ and  $\Sigmatd{\Instances^3, 
\Hypotheses^3, \hypothesis^3_0} + \Sigmatd{\Instances^4, \Hypotheses^4, \hypothesis^4_0} \geq 4$, which completes the proof. 


\end{proof}

}
}
{}
\end{document}